\documentclass[preprint,12pt]{elsarticle}

\usepackage{amssymb}

\usepackage[utf8]{inputenc} 
\usepackage[T1]{fontenc}    
\usepackage{url}            
\usepackage{booktabs}       
\usepackage{amsfonts}       
\usepackage{nicefrac}       
\usepackage{microtype}      
\usepackage{xcolor}         
\usepackage{multicol}
\usepackage{microtype}
\usepackage{graphicx}
\usepackage{caption}
\usepackage{subcaption}
\usepackage{booktabs} 
\usepackage{amsmath}
\usepackage{mathrsfs}
\usepackage{amssymb}
\usepackage{color}
\usepackage{xcolor}
\usepackage{multirow}
\usepackage{paralist}
\usepackage{verbatim}
\usepackage{algorithm}
\usepackage{algorithmic}
\usepackage{galois}
\usepackage{boxedminipage}
\usepackage{accents}
\usepackage{stmaryrd}
\usepackage[bottom]{footmisc}
\definecolor{light-gray}{gray}{0.85}
\usepackage{colortbl}
\usepackage{makecell}
\usepackage{hhline}
\usepackage{mathtools}
\usepackage{xcolor}         
\usepackage{bbold}
\usepackage[utf8]{inputenc} 
\usepackage[T1]{fontenc}    
\usepackage{url}            
\usepackage{booktabs}       
\usepackage{amsfonts}       
\usepackage{nicefrac}       
\usepackage{microtype}      
\usepackage{MnSymbol, graphicx}
\usepackage{xr}
\usepackage{enumitem}

\usepackage{pifont}

\newtheorem{theorem}{Theorem}[section]
\newtheorem{lemma}[theorem]{Lemma}

\newtheorem{definition}{Definition}[section]

\newtheorem{assumption}[theorem]{Assumption}

\usepackage{threeparttable}

\newcommand{\norm}[1]{\left\lVert#1\right\rVert}

\newcommand{\EE}{\mathbb{E}}

\newcommand{\x}{\mathbf x}

\newcommand{\ba}{\begin{array}}
\newcommand{\ea}{\end{array}}

\usepackage{diagbox}

\begin{document}

\begin{frontmatter}



\title{Safe and Balanced:  A Framework for Constrained Multi-Objective Reinforcement Learning}




\author{Shangding Gu$^{a*}$, Bilgehan Sel$^{b*}$,  Yuhao Ding$^{c*}$, Lu Wang$^d$, Qingwei Lin$^d$, Alois Knoll$^a$, Ming Jin$^b$}



\cortext[cor1]{Equation contribution. This manuscript is under actively development. We appreciate any constructive comments and
suggestions corresponding to   \textit{shangding.gu@tum.de}.} 



\affiliation{organization={Department of Computer Science, Technical University of Munich},
            country={Germany}}
            
 \affiliation{organization={Electrical and Computer Engineering Department, Virginia Tech},
            country={USA}}
            
 \affiliation{organization={Department of Industrial Engineering and Operations Research, UC Berkeley},
            country={USA}}
            
\affiliation{organization={Microsoft Research Asia},
            country={China}}

\begin{abstract}
In numerous reinforcement learning (RL) problems involving safety-critical systems, a key challenge lies in balancing multiple objectives while simultaneously meeting all stringent safety constraints.  To tackle this issue, we propose a primal-based framework that orchestrates policy optimization between multi-objective learning and constraint adherence. Our method employs a novel natural policy gradient manipulation method to optimize multiple RL objectives and overcome conflicting gradients between different tasks, since the simple weighted average gradient direction may not be beneficial for specific tasks' performance due to misaligned gradients of different task objectives. When there is a violation of a hard constraint, our algorithm steps in to rectify the policy to minimize this violation.  We establish theoretical convergence and constraint violation guarantees in a tabular setting. Empirically, our proposed method also outperforms prior state-of-the-art methods on challenging safe multi-objective reinforcement learning tasks. 

\end{abstract}




\end{frontmatter}


\clearpage
\tableofcontents
\clearpage
\section{Introduction} \label{sec: intro}

Reinforcement Learning (RL) has made significant strides  and is used widely in various domains~\cite{gu2023human, gu2022review}, e.g., robotics~\cite{gu2023safe, kober2013reinforcement}, autonomous driving~\cite{kiran2021deep, gu2022constrained}, large language model~\cite{ouyang2022training}, and finance~\cite{charpentier2021reinforcement}. However, a significant challenge arises when a policy must address multiple objectives within a single task or manage multiple tasks concurrently.  Direct optimization of scalarized objectives can lead to suboptimal performance, with the optimizer often struggling to make progress, resulting in a considerable decline in learning performance \cite{vandenhende2021multi}. A significant cause of this issue is the phenomenon of conflicting gradients \cite{yu2020gradient}. Here, gradients associated with different objectives may vary in scale, potentially leading the largest gradient to dominate the update. Moreover, they might point in different directions, i.e., $\nabla f_i\left(\pi\right)^{\top} \nabla f_j\left(\pi\right)<0, i \neq j, i, j \in[m]=$ $\{1, \ldots, m\}$, causing the performance of one objective to deteriorate during the optimization of another. While recent studies have shown that linear scalarization can be competitive \cite{kurin2022defense}, it may fall short when faced with safety-critical constraints. Indeed, ensuring the safe application of RL algorithms in real-world settings, especially those dealing with multiple objectives, is paramount \cite{huang2022constrained}. This study seeks to answer the key question:
\begin{center}
	\textbf{How can we balance each objective while ensuring safety constraints? } 
\end{center}
Addressing this problem, akin to a multi-dimensional tug of war, requires nuance. Each task is a team pulling in its own direction, yet confined by the boundaries of safety---a balancing act of objectives and safety. Inspired by this dynamic, we devise a comprehensive framework for Constrained Multi-Objective Reinforcement Learning (CMORL) using gradient manipulation and constraint rectification. It operates in three stages: (1) Estimating Q-functions from the existing policy. (2) If all constraints are satisfactorily met, the policy is updated via the manipulated natural policy gradient (NPG) of multiple objectives to minimize the gradient conflicts. (3) If not, the policy is updated following the NPG of the unsatisfied constraint.  These steps are iteratively repeated until convergence is achieved. 

Accompanying this framework, we provide a theoretical analysis, including convergence analysis and violation guarantee analysis. Using the insights from this analysis, we develop a practical algorithm to manage multi-objective RL while ensuring safety during learning. We further deploy our algorithm on safe multi-objective tasks in the MuJoCo environment~\cite{todorov2012mujoco} and compare our method with the state-of-the-art (SOTA) safe baseline, CRPO~\cite{xu2021crpo}, and SOTA safe multi-objective reinforcement learning methods, such as LP3~\cite{huang2022constrained}. Our experimental results suggest that our method outperforms CRPO and LP3 in striking a balance between reward performance and safety violation.


Our study offers several significant contributions to the field of safe multi-objective RL, which are delineated as follows:
(1) A novel framework for safe multi-objective RL, wherein a comprehensive analysis of both theoretical convergence and constraint violation guarantees is conducted.
(2) The development of a benchmark grounded in MuJoCo environments (Named \textit{Safe Multi-Objective MuJoCo}), aimed at scrutinizing the efficacy of safe multi-objective learning.
(3) The superior performance by our proposed method in terms of striking a balance between safety concerns and the accomplishment of multiple reward objectives, as evidenced across numerous challenging tasks within the realm of safe multi-objective RL.



\section{Related Work} \label{sec:related-work}

In recent years, several methods have been proposed to help deploy RL in real-world applications~\cite{gu2023safe, huang2022constrained, wu2021offline, gu2023human}, which try to solve the safe exploration problem~\cite{gu2022review} or satisfy multi-objective requirements during RL exploration~\cite{vithayathil2020survey} from the perspective of safe or multi-objective RL. 

\textbf{Safe Reinforcement Learning.}
Safe RL received remarkable attention since it can help address learning safety problems during RL deployment in real-world applications. Safe RL can be viewed as a constrained optimization problem~\cite{gu2022review}. For instance, several safe RL methods leverage Gaussian Processes to model safe state space during exploration~\cite{turchetta2016safe, berkenkamp2015safe, sui2015safe, wachi2018safe}. Different from the modeling safe state, some safe RL methods try to search the safe policy from constrained action space~\cite{chow2018lyapunov, chow2019lyapunov, koller2018learning, li2019temporal, marvi2021safe, fulton2018safe}, e.g., based on formal methods, the exploration action is verified via temporal logic verification during exploration~\cite{li2019temporal}. Moreover, by optimizing the average cumulative cost of each trajectory, some constrained policy optimization-based methods are proposed, such as CPO~\cite{achiam2017constrained}, PCPO~\cite{yang2020projection}, RCPO~\cite{tessler2018reward} and CRPO~\cite{xu2021crpo}.


\textbf{Multi-Objective Reinforcement Learning (MORL).} There are two settings in MORL~\cite{huang2022constrained}. One is just a single policy in the multi-objective optimization; another is a multi-policy set that satisfies multi-objective requirements.   Most MORL methods are developed based on the first setting, where a single policy needs to meet multi-objective conditions simultaneously~\cite{abdolmaleki2020distributional}. Furthermore, various multi-task learning methods are proposed to optimize the policy performance, such as the method of multi-task learning as a bargaining game~\cite{navon2022multi}, Cagrad~\cite{liu2021conflict}, the Multiple-Gradient descent Algorithm (MGDA)~\cite{desideri2012multiple}, PCGrad~\cite{yu2020gradient}. The second setting's methods try to learn a complete set of Pareto frontier and leverage a posteriors selection to satisfy multi-objective requirements~\cite{vamplew2011empirical}, such as MORL optimization based on a Manifold space to find a better solution of Pareto frontier~\cite{parisi2016multi, parisi2017manifold}.

The methods mentioned above tackle RL safety or multi-task requirements separately without considering both simultaneously. Our focus is on the task of achieving safe MORL, which involves ensuring exploration safety in multi-task RL settings. The most similar work to ours is the Learning Preferences and Policies in Parallel (LP3) algorithm~\cite{huang2022constrained}, which is proposed based on the Multi-Objective Maximum Posterior Policy optimization (MO-MPO)\cite{abdolmaleki2020distributional}, in which a supervised learning algorithm is used to learn the preferences, and then they train a policy based on the Lagrangian optimization. However, their method heavily depends on the Q-estimation, which may not accurately  present the safe preferences; the gradient conflict between each objective is not analyzed, and the convergence analysis and safety violation guarantee are not provided. In contrast with LP3~\cite{huang2022constrained}, we proposed a primal-based framework that can balance policy optimization between multi-task learning and constraint satisfaction based on the conflict-averse NPG, in which the conflict gradient  is analyzed between each objective performance, and convergence analysis and safety violation guarantee are provided based on  gradient manipulation and constraint rectification.

\section{Preliminaries and Problem Formulation} \label{sec: prelim}
\subsection{Multi-Objective RL (MORL)}

A MORL is a tuple $\left(\mathcal{S}, \mathcal{A}, \{r_i\}_{i=1}^m, \mathrm{P}, \rho, \gamma\right)$, where $\mathcal{S}$ and $\mathcal{A}$ are state and action spaces; $r_i: \mathcal{S} \times \mathcal{A} \rightarrow [0,r_{\max}]$ is the reward function; $m \geq 2$ denotes the number of tasks or objectives; $\mathrm{P}: \mathcal{S} \times \mathcal{A} \times \mathcal{S} \rightarrow[0,1]$ is the transition kernel, with $\mathrm{P}\left(s^{\prime} \mid s, a\right)$ denoting the probability of transitioning to state $s^{\prime}$ from previous state $s$ given action $a ; \rho: \mathcal{S} \rightarrow[0,1]$ is the initial state distribution; and $\gamma \in(0,1)$ is the discount factor. A policy $\pi\in\Pi: \mathcal{S} \rightarrow \mathcal{P}(\mathcal{A})$ is a mapping from the state space to the space of probability distributions over the actions, with $\pi(\cdot \mid s)$ denoting the probability of selecting action $a$ in state $s$. When the associated Markov chain $\mathrm{P}\left(s^{\prime} \mid s\right)=\sum_{\mathcal{A}} P\left(s^{\prime} \mid s, a\right) \pi(a \mid s)$ is ergodic, we denote $\mu_\pi$ as the stationary distribution of this MDP, i.e. $\int_{\mathcal{S}} \mathrm{P}\left(s^{\prime} \mid s\right) \mu_\pi(d s)=\mu_\pi\left(s^{\prime}\right)$. Moreover, we define the visitation measure induced by the policy $\pi$ as $\nu_\pi(s, a)=$ $(1-\gamma) \sum_{t=0}^{\infty} \gamma^t \mathrm{P}\left(s_t=s, a_t=a\right)$.

For a given policy $\pi$ and a reward function $r_i$, we define the state value function as $V^\pi_i(s)=\mathbb{E}\left[\sum_{t=0}^{\infty} \gamma^t r_i\left(s_t, a_t\right) \mid s_0=s, \pi\right]$, the state-action value function as $Q^\pi_i(s, a)=$ $\mathbb{E}\left[\sum_{t=0}^{\infty} \gamma^t r_i\left(s_t, a_t\right) \mid s_0=s, a_0=a, \pi\right]$, the advantage function as $A^\pi_i(s, a)=Q^\pi_i(s, a)-V^\pi_i(s)$, and the expected total reward function $f_i(\pi)=\mathbb{E}\left[\sum_{t=0}^{\infty} \gamma^t r_i\left(s_t, a_t\right)\right]=\mathbb{E}_{\rho}\left[V^\pi_i(s)\right]=$ $\mathbb{E}_{\rho \cdot \pi}\left[Q^\pi_i(s, a)\right]$.

In MORL, we aim to find a single optimal policy that maximizes multiple expected total reward  functions simultaneously, termed as
\begin{align} \label{eq: MORL}
\max_{\pi \in \Pi} \ \boldsymbol{F}(\pi)=\left(f_1(\pi), \ldots, f_m(\pi)\right)^{\top}.
\end{align}

\subsection{Constrained Multi-Objective RL (CMORL)}
The CMORL problem refers to a formulation of MORL that involves additional \textit{hard constraints} that restrict the allowable policies. The constraints take the form of costs that the agent may incur when taking actions at certain states, denoted by the functions $r_{m+1}, \cdots, r_{m+p}$. Each of these cost functions maps a tuple $\left(s, a\right)$ to a corresponding cost value. The function $f_{m+i}(\pi)$ represents the expected total cost incurred by the agent with respect to cost function $r_{m+i}$. The objective of the agent in CMORL is to solve a multi-objective RL problem subject to the aforementioned hard constraints:
\begin{align} \label{eq: CMO-MDP}
\max_{\pi\in\Pi} \ \boldsymbol{F}(\pi), \text { s.t. } f_i(\pi) \leq c_i, \forall i=m+1, \cdots, m+p,
\end{align}
where $c_i$ is a fixed limit for the $i$-th constraint. 
We define the safety set $\Pi_{\text{safe}} = \left\{ \pi \in \Pi \mid  f_i\left(\pi\right) \leq c_i, \forall i=m+1, \cdots, m+p\right\}$, and the optimal policy $\pi^*=\mathop{\arg\max}\limits_{\pi \in \Pi_{\text{safe}}} \boldsymbol{F}(\pi)$ for CMORL in \eqref{eq: CMO-MDP}.
 In practice, a convenient way to solve RL is to parameterize the policy and then iteratively optimize the policy over the parameter space. Let $\left\{\pi_w: \mathcal{S} \rightarrow \mathcal{P}(\mathcal{A}) \mid w \in \mathcal{W}\right\}$ be a parameterized policy class, where $\mathcal{W}$ is the parameter space. Then, the problem in \eqref{eq: CMO-MDP} can be written as 
$$
\max_{w \in \mathcal{W}} \ \boldsymbol{F}(\pi_w)  \ \text {, s.t. } \ f_i\left(\pi_w\right) \leq c_i, \ \forall i=m+1, \cdots, m+p.
$$


In CMORL, we extend the notion of the Pareto frontier, which is defined to compare the policies, from the unconstrained MDP \cite{zhou2022convergence} to the safety-constrained MDP.

\begin{definition}[Safe Pareto Frontier]
For any two policies $\pi, \pi^{\prime} \in \Pi$, we say that $\pi$ dominates $\pi^{\prime}$
if $f_i(\pi) \leq f_i\left(\pi^{\prime}\right)$ for all $i$, and there exists one $i$ such that $f_i(\pi)<f_i\left(\pi^{\prime}\right)$; otherwise, we say that $\pi$ does not dominate $\pi^{\prime}$.
A solution $\pi^* \in \Pi_{\text{safe}}$ is called safe Pareto optimal if it is not dominated by any other safe policy in $\Pi_{\text{safe}}$. The set of all safe Pareto optimal policies is the safe Pareto frontier.
\end{definition}

The goal of CMORL is to find a safe Pareto optimal policy. However, the simultaneous learning of numerous tasks presents a complex optimization issue due to the involvement of multiple objectives \cite{vandenhende2021multi}. The most popular multi-objective/multi-task formulation in practice is the linear scalarization of all tasks given relative preferences for each task $\xi_i, i\in [m]$:
$$
\max_{w \in \mathcal{W}} \boldsymbol{\xi}^\top \boldsymbol{F}(\pi_w) \ \text {, s.t. } \ f_i\left(\pi_w\right) \leq c_i, \ \forall i=m+1, \cdots, m+p.
$$   
Even when this linear scalarization formulation gives exactly the true objective, directly optimizing it could lead to undesirable performance due to conflicting gradients, dominating gradients, and high curvature \cite{yu2020gradient}.

In this paper, we aim to find a safe Pareto optimal solution using the gradient-based method by starting from an arbitrary initialization policy $\pi_t$ and iteratively finding the next policy $\pi_{t+1}$ by moving against a direction $\boldsymbol{d}_t$ with step size $\eta$, i.e., $\pi_{t+1}=\pi_t+\eta \boldsymbol{d}_t$. The design of the direction $\boldsymbol{d}_t$ is the key to the success of CMORL. A good direction $\boldsymbol{d}_t$ should enable us to move from a policy $\pi_{t+1}$ to $\pi_{t}$ such that either  $\pi_{t+1}$ dominates $\pi_{t}$  or $\pi_{t+1}$ improves the hard constraint satisfaction compared with $\pi_{t}$, or both.

\section{Constraint-Rectified Multi-Objective Policy Optimization (CR-MOPO)} \label{sec: algorithm}
In this section, we introduce a general framework called CR-MOPO which decomposes safe Pareto optimal policy learning into three sub-problems and iterates until convergence:
\begin{enumerate}
    \item Policy evaluation: estimate Q-functions given the current policy.
    \item Policy improvement for the multi-objectives: update policy based on the manipulated NPG of multi-objectives when constraints are all approximately satisfied.
    \item Constraint rectification: update policy based on the NPG of an unsatisfied constraint when constraints are not all approximately satisfied.
\end{enumerate}

Algorithm \ref{alg:CR-MOPO framework} summarizes this three-step constrained multi-objective policy improvement framework and Algorithm \ref{alg:MOPO-CR} provides a concrete realization with our novel conflict-averse NPG method. Note, based on our theoretical guarantee on the time-average convergence, policy $\pi_{\text{out}}$ can be uniformly chosen from $\mathcal{N}_0$, the detail proof is provided in appendix~\ref{sec: Proof}. To ease the presentation and better illustrate the main idea, we will focus on the tabular MDP setting in this section. The extension to the more practical setting of deep RL will be discussed in appendix~\ref{sec: Proof}.

\begin{algorithm}[ht!] 
\caption{\textbf{CR-MOPO}: Constraint-Rectified Multi-Objective Policy Optimization Framework}
\label{alg:CR-MOPO framework}
\begin{algorithmic}[1]
\STATE \textbf{Inputs}: initial parameter $\pi_{w_0}$, empty set $\mathcal{N}_0$.
\FOR{$t = 0, \dots,T-1$}
\STATE Policy evaluation under $\pi_{w_t}$ for all objectives and constraints.
\IF{constraints are all satisfied}
\STATE Add $\pi_{w_t}$ into set $\mathcal{N}_0$.
\STATE Compute the multi-objective policy update direction $\boldsymbol{d}$ and update policy using $\boldsymbol{d}$.
\ELSE
\STATE Choose any unsatisfied constraint $i_t$ and update policy towards minimize $f_{i_t}(\pi_{w_t})$.
\ENDIF
\ENDFOR{}
\STATE \textbf{Outputs}: $\pi_{\text{out}}$ uniformly chosen from $\mathcal{N}_0$.
\end{algorithmic}
\end{algorithm}

\subsection{Policy Evaluation}
In this step, we aim to learn Q-functions that can effectively evaluate the preceding policy $\pi_t$. To achieve this, we train individual Q-functions for each objective and constraint. In principle, any Q-learning algorithm can be used, as long as the target Q-value is 
computed with respect to $\pi_t$. 

\paragraph{Temporal difference (TD) learning.} In TD learning, each iteration takes the form of 
\begin{align} \label{eq: TD learning}
    Q^{\pi_w}_{i,k+1}(s,a) = Q^{\pi_w}_{i,k} + \ell_k \left[ r_i(s,a) + \gamma Q^{\pi_w}_{i,k}(s^\prime,a^\prime)  - Q^{\pi_w}_{i,k}(s,a) \right],
\end{align}
where $s \sim \mu_{\pi_w}, a\sim  \pi_w(s), s^\prime \sim \mathrm{P}(\cdot\mid s,a), a^\prime \sim \pi_w(s^\prime)$, and $\ell_k$ is the learning rate. It has been shown in \cite{bhandari2018finite,dalal2018finite}
that the iteration in \eqref{eq: TD learning} converges to the fixed point which is the state-action value $Q_{i}^{\pi_w}$. After performing $K_{\text{TD}}$ iterations
of \eqref{eq: TD learning}, we let the estimation $\bar{Q}_{i}(s,a)=Q^{\pi_w}_{i,K_{\text{TD}}}(s,a)$.

\paragraph{Unbiased Q-estimation.} 
To obtain an unbiased estimation of the state-action value \cite{zhang2020global},  we can perform Monte-Carlo rollouts for a trajectory with the horizon $H \sim \text{Geom}(1-\gamma^{1/2})$, where $\text{Geom}(x)$ denotes a geometric distribution with parameter $x$,  and estimate the state-action value function along the trajectory $(s_0,a_0,\ldots,s_H,a_H)$ as follows:
\begin{align}\label{eq: unbiased Q-estimation}
\bar{Q}_{i}(s_0,a_0)  = r_i(s_0,a_0) + \sum_{h=1}^{H} \gamma^{h/2} r_i(s_h,a_h).
\end{align}

\subsection{Policy Improvement for Multi-Objectives}
\subsubsection{Conflict-Averse Natural Policy Gradient (CA-NPG)}
The policy gradient \cite{sutton1999policy} of the value function $f_i\left(\pi_w\right)$ has been derived as $\nabla f_i\left(\pi_w\right)=\mathbb{E}\left[Q^{\pi_w}_i(s, a) \phi_w(s, a)\right]$, where $\phi_w(s, a):=\nabla_w \log \pi_w(a \mid s)$ is the score function. 
However, the standard policy gradient does not effectively reflect the statistical manifold (the family of probability distributions that represents the policy function) that the policy operates on. To prevent the policy itself from changing too much during an update, we need to consider how sensitive the policy is to parameter changes. 



Thus, in the multi-objectives policy optimization, 
we aim to choose an update direction $\boldsymbol{d}$ to increase every individual value function while imposing the constraint on the allowed changes of an update in terms of the KL divergence of the policy. To do so, we consider the following constrained optimization problem:
\begin{align} \label{problem: before taylor}
    \max_{d: D_{\text{KL}}\left( \pi_w\mid \pi_{w+ \boldsymbol{d}}\right)\leq \epsilon_0} \min_{i\in [m]}\left\{ \xi_i \left(f_i(w+ \boldsymbol{d}) - f_i(w) \right)  \right\}
\end{align}
where $\epsilon_0$ is the pre-defined threshold for allowed policy changes.
By using the first-order Taylor approximation for the value improvement, the second-order  Taylor approximation for the KL divergence constraint and the Lagrangian relaxation, the problem \eqref{problem: before taylor} can be rewritten as
\begin{align} \label{problem: after taylor}
    \max_{d} \min_{i\in [m]}\left\{ \xi_i \nabla f_i(w)^{\top} \boldsymbol{d} - \frac{\psi_1}{2} \boldsymbol{d}^\top \tilde{F}(w) \boldsymbol{d}\right\}
\end{align}
where $\psi_1>0$ is a pre-specified hyper-parameter to control the allowed changes in policy space and $\tilde{F}(w)$ is the Fisher information matrix defined as $\tilde{F}(w) = \left.\nabla_{w^\prime}^2
D_{\text{KL}}\left( \pi_w\mid \pi_{w^\prime}\right)\right|_{w^\prime=w}
=\mathbb{E}_{\nu_{\pi_w}}\left[\phi_w(s, a) \phi_w(s, a)^{\top}\right]$. For a single objective $f_i$, the solution of \eqref{problem: after taylor} leads to the well-known NPG update \cite{kakade2001natural} which is defined as $\tilde{F}(w)^{\dagger} \nabla f_i\left(\pi_w\right)$. Note that TRPO \cite{schulman2015trust} can be viewed as the NPG approach with adaptive stepsize.

With the above problem formulation, we aim to find an update direction that minimizes the gradient conflicts.
Furthermore, inspired by the recent advances in gradient manipulation method \cite{liu2021conflict} which looks for the best update direction within a local ball centered at the weighted averaged gradient, we also constraint search region for the common direction as a circle around the weighted average policy gradient $\boldsymbol{v}_0=\sum_{i=1}^m \xi_i \nabla f_i(w)$. This yields \textbf{Conflict-Averse Natural Policy Gradient (CA-NPG)} which determines the update direction $\boldsymbol{d}$ by solving the following optimization problem
\begin{align} \label{eq: CA-NPG}
\underset{\boldsymbol{d}}{ \max } \min _{i \in[m]}  \ \left\{\xi_i \nabla f_i(w)^{\top} \boldsymbol{d} - \frac{\psi_1}{2} \boldsymbol{d}^\top \tilde{F}(w) \boldsymbol{d} - \frac{\psi_2}{2} \left\|\boldsymbol{d}-\boldsymbol{v}_0\right\|^2  \right\},
\end{align}
where $\psi_2>0$ is a pre-specified hyper-parameter that controls the deviation from the weighted average policy gradient $\boldsymbol{v}_0$.
Furthermore, notice that
$\min_i \xi_i \nabla f_i(w)^{\top} \boldsymbol{d} = \min_{\boldsymbol{\theta} \in S_m} \sum_{i\in[m]}  \theta_i \xi_i \nabla f_i(w)^{\top} \boldsymbol{d} $, where $\boldsymbol{\theta}=(\theta_1,\ldots,\theta_m)$ and $S_m=\left\{\boldsymbol{\theta}: \sum_{i=1}^m \theta_i=1, \theta_i\geq 0 \right\}$. Denote $\nabla f_{\boldsymbol{\theta}}(w) = \sum_{i\in[m]}  \theta_i \xi_i \nabla f_i(w)$.
The  objective in \eqref{eq: CA-NPG} can be written as 
\begin{align*}
\underset{\boldsymbol{d}}{ \max } \min _{\boldsymbol{\theta} \in S_m} \left\{\nabla f_{\boldsymbol{\theta}}(w)^{\top} \boldsymbol{d} - \frac{\psi_1}{2} \boldsymbol{d}^\top \tilde{F}(w) \boldsymbol{d} - \frac{\psi_2}{2} \left\|\boldsymbol{d}-\boldsymbol{v}_0\right\|^2  \right\}.
\end{align*}
Since the above objective is concave with respect to $\boldsymbol{d}$ and linear with respect to $\boldsymbol{\theta}$, by switching the min and max, we reach
the dual form without changing the solution:
\begin{align*}
 \min _{\boldsymbol{\theta} \in S_m}  \underset{\boldsymbol{d}}{ \max }  \left\{\nabla f_{\boldsymbol{\theta}}(w)^{\top} \boldsymbol{d} - \frac{\psi_1}{2} \boldsymbol{d}^\top \tilde{F}(w) \boldsymbol{d} - \frac{\psi_2}{2} \left\|\boldsymbol{d}-\boldsymbol{v}_0\right\|^2  \right\}.
\end{align*}
After a few steps of calculus, we derive the following optimization problem with respect to the variable $\boldsymbol{\theta}$:
\begin{align*}
\boldsymbol{\theta}^* =  \arg  \min_{\boldsymbol{\theta} \in S_m} & \nabla f_{\boldsymbol{\theta}}^{\top} \left(\psi_1 \tilde{F}+\psi_2 I\right)^{-1} \left(\nabla f_{\boldsymbol{\theta}} +\psi_2\boldsymbol{v}_0\right) \\
 &- \frac{\psi_1}{2} \left(\nabla f_{\boldsymbol{\theta}} +\psi_2\boldsymbol{v}_0\right)^\top \left(\psi_1 \tilde{F}+\psi_2 I\right)^{-1} \\ & \tilde{F} \left(\psi_1 \tilde{F}+\psi_2 I\right)^{-1} \left(\nabla f_{\boldsymbol{\theta}} +\psi_2\boldsymbol{v}_0\right) \\
 &-  \frac{\psi_2}{2}\left\|\left(\psi_1 \tilde{F}+\psi_2 I\right)^{-1} \left(\nabla f_{\boldsymbol{\theta}} +\psi_2\boldsymbol{v}_0\right)\right\|^2 ,
\end{align*}
and the optimal update direction is given by
\begin{align} \label{eq: CA-NPG simplified}
&\boldsymbol{d}^*
\coloneqq \boldsymbol{\lambda}^\top  \nabla \boldsymbol{F}
=\sum_{i\in[m]} \lambda_i \nabla f_i, \quad \\
& \text{ where } \lambda_i = \xi_i \left( \theta_i^*+ \psi_2 \right) \left(\psi_1 \tilde{F}+\psi_2 I\right)^{-1}. \nonumber
\end{align}

\subsubsection{Correlation-Reduction for Stochastic Gradient Manipulation}
In practice, we only obtain noisy policy gradient feedback  $\widehat{\nabla \boldsymbol{F}}(w_t)$, where the stochastic noise is due to the finite sampled trajectories for the estimation of $Q^{\pi_{w_t}}_i$. It has been shown in \cite{zhou2022convergence} that the gradient manipulation methods may fail to converge to a Pareto optimal solution under the stochastic setting.
This convergence gap is mainly caused by the strong correlation between the  weights $\boldsymbol{\lambda}_t$
and the stochastic gradients $\widehat{\nabla \boldsymbol{F}}(w_t)$ which yields a biased composite gradient.
To address this issue in {CA-NPG}, we consider two conditions. The first is that the NPG estimator variance asymptotically converges to $0$. For example, this can be achieved by estimating $Q^{\pi_{w_t}}_i$ using TD learning in \eqref{eq: TD learning} with sufficiently large $K_{\text{TD}}$.
The second is to reduce the variances of $\boldsymbol{\lambda}_\tau$ by adopting a momentum mechanism \cite{zhou2022convergence} with
coefficient $\alpha_t$ on the update of composite weights 
\begin{equation} \label{eq: momentum mechanism}
    \widehat{\boldsymbol{\lambda}}_\tau= \alpha_\tau \widehat{\boldsymbol{\lambda}}_{\tau-1} + (1-\alpha_\tau){\boldsymbol{\lambda}}_\tau,
\end{equation}
where ${\boldsymbol{\lambda}}_\tau$ is computed by {CA-NPG} algorithms .

\subsection{Constraint Rectification}
We then check whether there exists a hard constraint $i\in \{m+1, \ldots, m+p\}$ such that the (approximated) constraint function violates the condition. If so, we take one-step update of the
policy using NPG towards minimizing the corresponding constraint
function $f_{i}(\pi_{w_t})$ to enforce the constraint:
\begin{align*}
w_{t+1} =w_t - \eta \tilde{F}(w)^{\dagger} \nabla f_i\left(\pi_w\right).
\end{align*}
If multiple
constraints are violated, we can choose to minimize any
one of them. Otherwise, we take one update of the policy towards maximizing the multi-objectives.

\subsection{Comparison with Learning Preferences and Policies in Parallel (LP3) \cite{huang2022constrained}}
Compared with SOTA safe multi-objective RL method, LP3, our new framework is different in both multi-objective optimization and hard constraint satisfaction. Firstly, LP3 chooses MO-MPO \cite{abdolmaleki2020distributional} as the multi-objective optimizer which encodes the objective preferences in a scale-invariant way through the allowed KL divergence for the updated policy using each objective. On the other hand, our multi-objective optimization method is based on linear scalarization coupled with novel NPG manipulation which encodes the preference in a more straightforward way and is tailored to RL to address the conflicting gradients and dominating gradients.
Secondly, LP3 can be regarded as a primal-dual approach where the additional dual variables are introduced as the adaptive weights for the constraints. This relaxes hard constraints in safe multi-objective RL problems to new objectives where the associated weights are adjusted based on the constraint violation conditions. On the other hand, our primal-based method does not suffer from extra hyperparameter tuning and dual update and can be implemented as easily as unconstrained policy optimization algorithms. 


\begin{algorithm}[ht!] 
\caption{{CR-MOPO} with CA-NPG as multi-objective optimizer.}
\label{alg:MOPO-CR}
\begin{algorithmic}[1]
\STATE \textbf{Inputs}: initial parameter $w_0$, empty set $\mathcal{N}_0$, $\tau=0$.
\FOR{$t = 0, \dots,T-1$}
\STATE Policy evaluation under $\pi_{w_t}: \Bar{Q}^t_i(s,a) \approx Q^{\pi_{w_t}}_i(s,a)$  for all $ i=1,\ldots,m+p$.
\STATE Collect pairs $(s^j, a^j) \in \mathcal{B}_t \sim \rho \cdot \pi_{w_t}$, compute constrain estimation $\Bar{J}_{i,\mathcal{B}_t} = \sum_{j \in \mathcal{B}_t} \frac{1}{\left| \mathcal{B}_t \right|}\Bar{Q}_t^i(s^j,a^j) $ for all $ i=1,\ldots,m+p$, where $j$ is the index for the sampled pairs in $\mathcal{B}_t$.
\IF{$\Bar{J}_{i,\mathcal{B}_t}\leq c_i +\beta$ for all $i=m+1,\ldots,m+p$}
\STATE $\tau\leftarrow \tau+1$;  add $w_t$ into set $\mathcal{N}_0$.
\STATE Compute the weights ${\boldsymbol{\lambda}}_\tau$ using \eqref{eq: CA-NPG simplified} and reduce the correlation by \eqref{eq: momentum mechanism}.
\STATE Compute the multi-objective policy gradient $\boldsymbol{d}_\tau =\widehat{\boldsymbol{\lambda}}_\tau^\top $$\widehat{\nabla \boldsymbol{F}}(w_t)$.
\STATE Take one-step policy update: $w_{t+1}=w_{t} +\eta \boldsymbol{d}_\tau$.
\ELSE
\STATE Choose any $i_t\in \left\{m+1,\dots, m+p \right\}$ such that $\Bar{J}_{i_t, \mathcal{B}_t} > c_{i_t} + \beta$.
\STATE Take one-step policy update towards minimize $J_{i_t}(w_t)$: 
$w_{t+1} \leftarrow w_t - \eta \tilde{F}(w_t)^{\dagger} \nabla f_i\left(w_t\right)$.
\ENDIF
\ENDFOR{}
\STATE \textbf{Outputs}: $w_{\text{out}}$ uniformly chosen from $\mathcal{N}_0$.
\end{algorithmic}
\end{algorithm}


\section{Theoretical Analysis} \label{sec: Convergence}
In this section, we establish the convergence and the constraint violation guarantee for CR-MOPO in the tabular settings under the softmax parameterization and CA-NPG. 
In the tabular setting, we consider the softmax parameterization. For any $w \in \mathbb{R}^{|\mathcal{S}| \times|\mathcal{A}|}$, the corresponding softmax policy $\pi_w$ is defined as
$
\pi_w(a \mid s):=\frac{\exp (w(s, a))}{\sum_{a^{\prime} \in \mathcal{A}} \exp \left(w\left(s, a^{\prime}\right)\right)},  \forall(s, a) \in \mathcal{S} \times \mathcal{A} .
$
Clearly, the policy class defined above is complete, as any stochastic policy in the tabular setting can be represented in this class. 
Since the adaptive weights $\boldsymbol{\lambda}_t$ of {CA-NPG} may not be constrained in the probability simplex. Hence, we consider the following mild assumption on the boundedness of $\boldsymbol{\lambda}_t$.
\begin{assumption}\label{ass: lambd bound}
For the {CA-NPG} mechanism, there exists finite constants $B_1>0$ and $B_2>0$ such that $0\leq \lambda_t^i\leq B_1, \sum_{i=1}^m \lambda_t^i \geq B_2 $ for all $t=1,\ldots, T$, $i=1,\ldots,m$.
\end{assumption}

For multi-objective optimization, if there \textit{exists} $\boldsymbol{\lambda}^*\in S_m$ such that $w^*=\arg\min _{w}\boldsymbol{\lambda}^{* \top} \boldsymbol{F}\left(\pi_{w}\right)$, then $w^*$ is (weak) Pareto optimal [Theorem 5.13 and Lemma 5.14 in \cite{john2004vector}]. Thus, we use $\min _{\boldsymbol{\lambda}^* \in S_m}\left(\boldsymbol{\lambda}^{* \top} \boldsymbol{F}\left(\pi^*\right)-\boldsymbol{\lambda}^{* \top} \boldsymbol{F}\left(\pi_{w_\text{out}}\right)\right)$ to measure the convergence to a Pareto optimal policy where 
minimization operator of $\boldsymbol{\lambda}^*$ is from the existence condition.
The following theorem characterizes the convergence rate
of Algorithm \ref{alg:MOPO-CR} in terms of the Pareto optimal policy convergence and hard constraint violations. The proof can be found in the appendix~\ref{sec: Proof}.

\begin{theorem} \label{thm: main}
Consider Algorithm \ref{alg:MOPO-CR} in the tabular setting with softmax policy parameterization and any policy initialization $w_0 \in \mathcal{R}^{\left|\mathcal{S} \right| \left|\mathcal{A} \right|}$. Let the tolerance be $\beta = \mathcal{O} \left(\frac{m B_1 \sqrt{\left|\mathcal{S}\right| \left| \mathcal{A} \right|} }{(1-\gamma)^2 \sqrt{T}} \right)$ and the learning rate for the CA-NPG and NPG be $\eta = \mathcal{O} \left(   \frac{(1-\gamma)^2}{m B_1 \sqrt{|\mathcal{S}||\mathcal{A}|T  }  }\right)$. Depending on the choice of the state-action value estimator, the following holds.
\begin{itemize}
    \item If TD-learning in \eqref{eq: TD learning} is used for policy evaluation with \resizebox{!}{0.45cm}{$K_{\text{TD}}=\widetilde{\mathcal{O}} \left( \left( \frac{T}{(1-\gamma)^2 \left|\mathcal{S}\right| \left| \mathcal{A} \right|} \right)^{\frac{1}{\sigma}} \right)$} , $\ell_k= \mathcal{O}(\frac{1}{k^\sigma}) $ and $\alpha_\tau=0 $ for $0<\sigma<1$, then with probability $1-\delta$, we have
    \begin{align*}
    &\mathbb{E}\left[\min _{\boldsymbol{\lambda}^* \in S_m}\left(\boldsymbol{\lambda}^{* \top} \boldsymbol{F}\left(\pi^*\right)-\boldsymbol{\lambda}^{* \top} \boldsymbol{F}\left(\pi_{w_\text{out}}\right)\right)\right] \leq 
 \frac{\beta}{B_2}, \quad \\ & \mathbb{E}\left[f_i\left(\pi_{w_{\text{out }}}\right)\right]-c_i \leq \beta, 
\end{align*}
for all $i=\{m+1,\ldots, m+p\}$,
where the expectation is taken only with respect to selecting $w_{\text{out}}$ from $\mathcal{N}_0$.
    \item If unbiased Q-estimation in \eqref{eq: unbiased Q-estimation} is used for policy evaluation with $\alpha_\tau \geq 1- \frac{1-\gamma}{{m \tau \sqrt{\left|\mathcal{S}\right| \left| \mathcal{A} \right|}}}$,
    we have 
\begin{align*}
    &\mathbb{E}\left[\min _{\boldsymbol{\lambda}^* \in S_m}\left(\boldsymbol{\lambda}^{* \top} \boldsymbol{F}\left(\pi^*\right)-\boldsymbol{\lambda}^{* \top} \boldsymbol{F}\left(\pi_{w_\text{out}}\right)\right)\right]  \leq \frac{\beta}{B_2}, \quad \\ & \mathbb{E}\left[f_i\left(\pi_{w_{\text{out }}}\right)\right]-c_i \leq \beta, 
\end{align*}
for all $i=\{m+1,\ldots, m+p\}$,
where the expectation is taken with respect to selecting $w_{\text{out}}$ from $\mathcal{N}_0$ and the randomness of $Q^i_{\pi_{w_t}}$ estimation.
\end{itemize}
\end{theorem}
As shown in Theorem \ref{thm: main}, our method is guaranteed to find a safe Pareto optimal policy under some mild conditions while there is no convergence guarantee for LP3 \cite{huang2022constrained}. Furthermore, results for unbiased Q-estimation imply that the correlation reduction mechanism could help the convergence even if we do not have an asymptotically increasing trajectory for policy evaluation, such as $K_{\text{TD}} = \widetilde{ \mathcal{O} }(T^{1/\sigma})$ in TD-learning.

\section{Experiments} \label{sec: exp}

\textbf{Environment Settings.} 
We designed a benchmark, termed \textit{Safe Multi-Objective MuJoCo}\footnote{\url{https://github.com/SafeRL-Lab/Safe-Multi-Objective-MuJoCo}}, for the purpose of scrutinizing our algorithms within the context of the MuJoCo framework~\cite{brockman2016openai, todorov2012mujoco}. A comprehensive overview of this benchmark can be found in Appendix \ref{append:environment-settings}, where Safe Multi-Objective HalfCheetah, Safe Multi-Objective Hopper, Safe Multi-Objective Humanoid, Safe Multi-Objective Swimmer, Safe Multi-Objective Walker, Safe Multi-Objective Pusher  are introduced to evaluate the effectiveness of our methods. Two of the Safe Multi-Objective MuJoCo environments, Safe Multi-Objective Humanoid and Pusher, are introduced in Figure~\ref{fig:cmorl-cmorl-crpos-safe-multi-task-halfcheetah-different-limit-0Dot005} (a) and (b). Please see Appendix \ref{append:environment-settings} for the detailed environment settings.

\textbf{CR-MOPO Performance on Challenging Safe Multi-Objective Environments.}
As shown in Figure~\ref{fig:safe-multi-task-different-tasks},  we conduct experiments with our algorithm, CR-MOPO, on challenging Safe Multi-Objective MuJoCo Environments. The each step's cost limit of HalfCheetah-v4 is $0.1$, HalfCheetah-v4-different-limit is $0.3$, Humanoid-v4 is  $0.9$, Humanoid-dm is $1.5$, Walke-v4 is $0.03$,  Hopper-v4 is $0.03$, Pusher-v4 is $0.49$, Swimmer is $0.049$. We optimize safety violations after $40$ Epochs except for the Humanoid-dm task; in the Humanoid-dm task, we optimize the safety violations after 5 Epochs. On all the challenging tasks, the experiment results indicate our method can guarantee each task's reward monotonic improvement while ensuring safety. Please see Appendix \ref{append:environment-settings} for more experiments.

\begin{figure}[htbp!]
 \centering
\subcaptionbox{}
 {
  \includegraphics[width=0.45\linewidth]{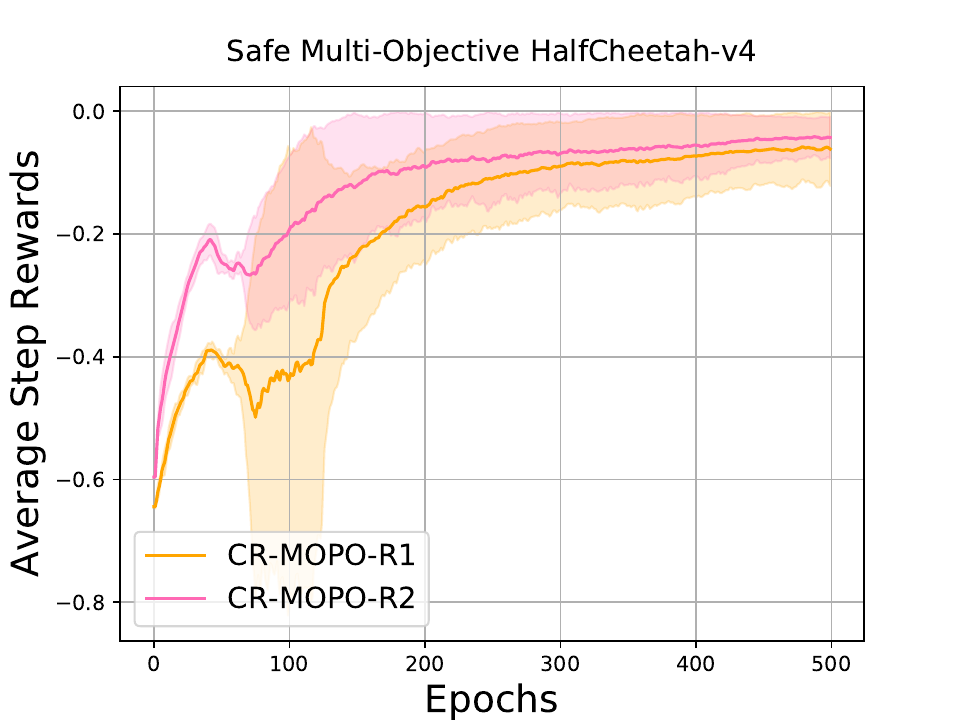}
  }
   \subcaptionbox{}
  {
\includegraphics[width=0.45\linewidth]{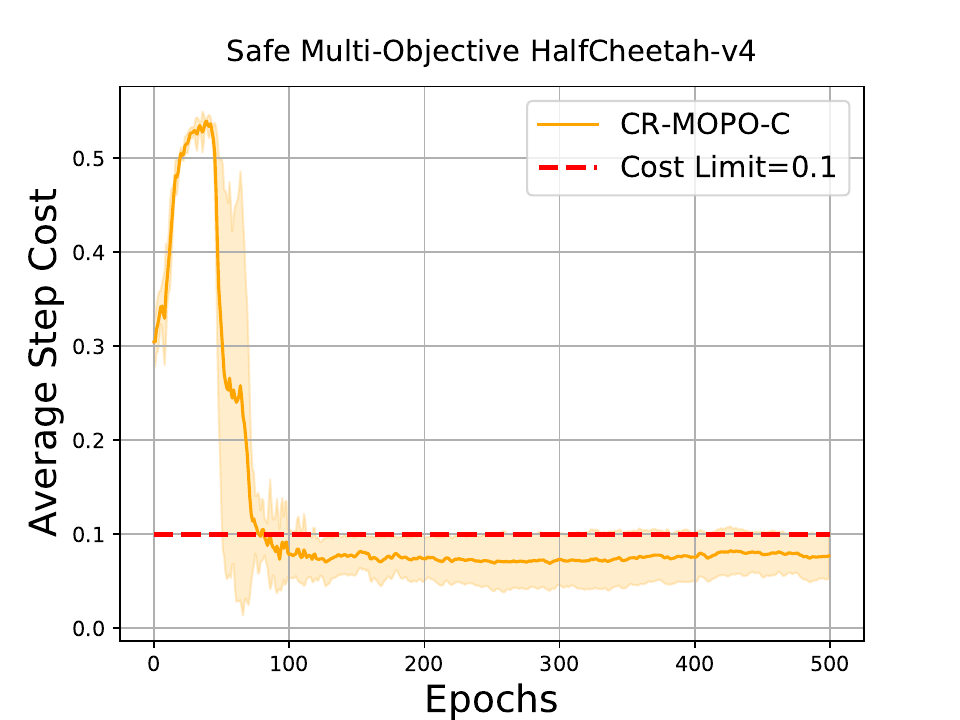}
}
\subcaptionbox{}
 {
  \includegraphics[width=0.45\linewidth]{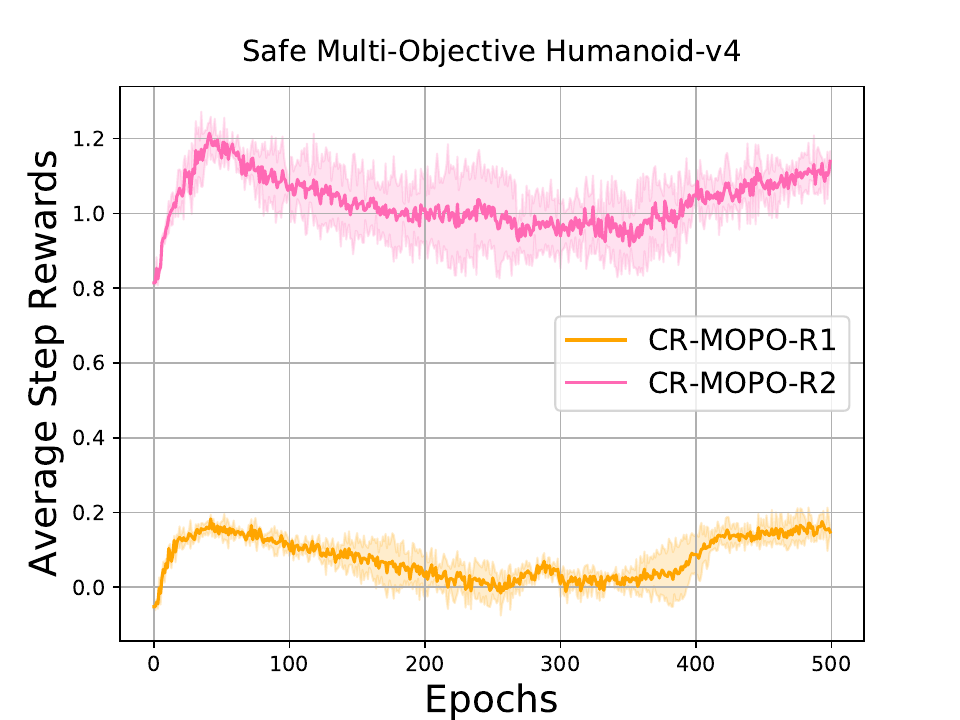}
  }
   \subcaptionbox{}
  {
\includegraphics[width=0.45\linewidth]{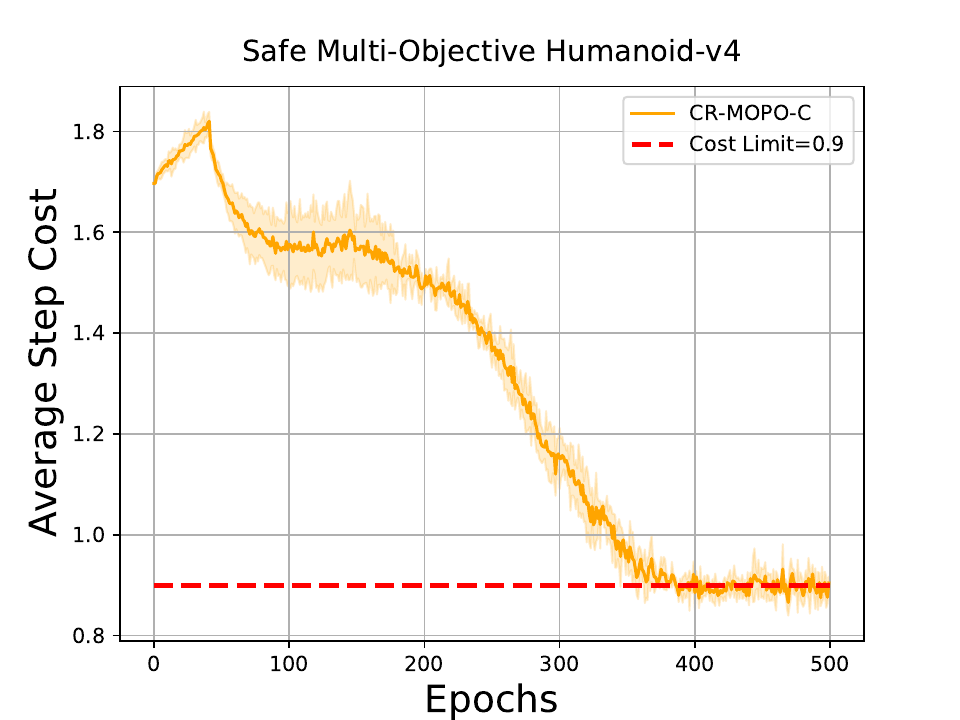}
}
 \subcaptionbox{}
 {
  \includegraphics[width=0.45\linewidth]{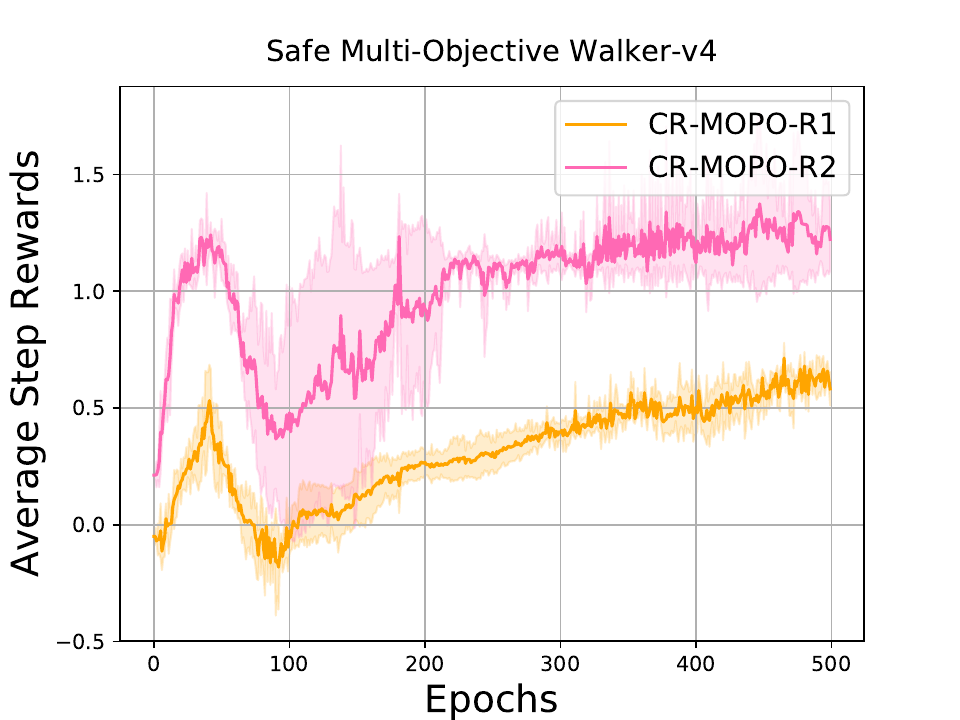}
  }
   \subcaptionbox{}
  {
\includegraphics[width=0.45\linewidth]{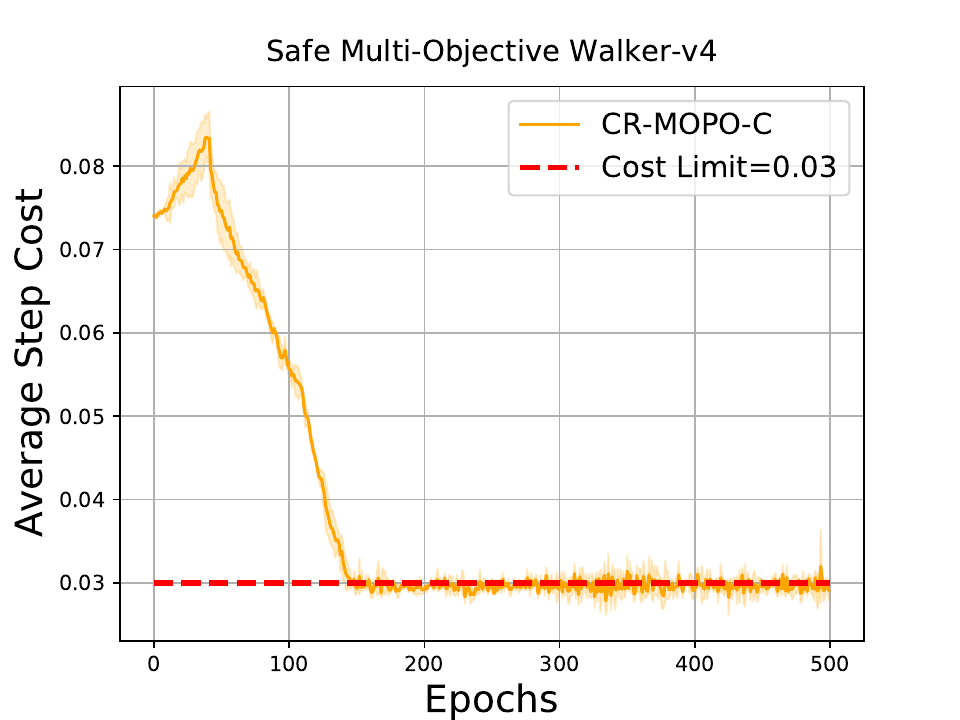}
}
\subcaptionbox{}
 {
  \includegraphics[width=0.45\linewidth]{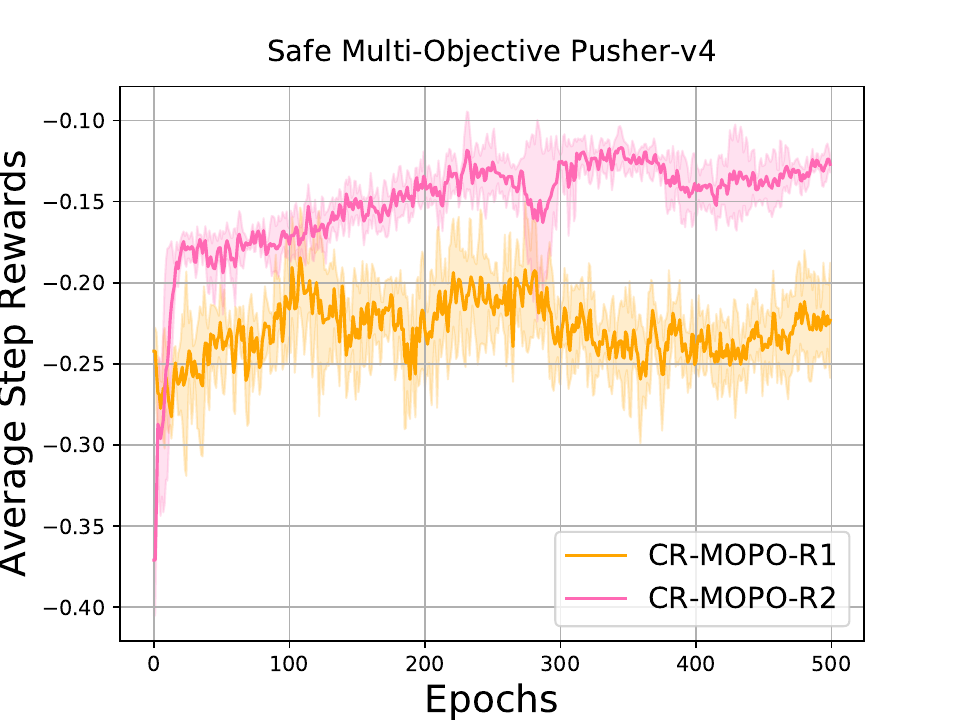}
  }
   \subcaptionbox{}
  {
\includegraphics[width=0.45\linewidth]{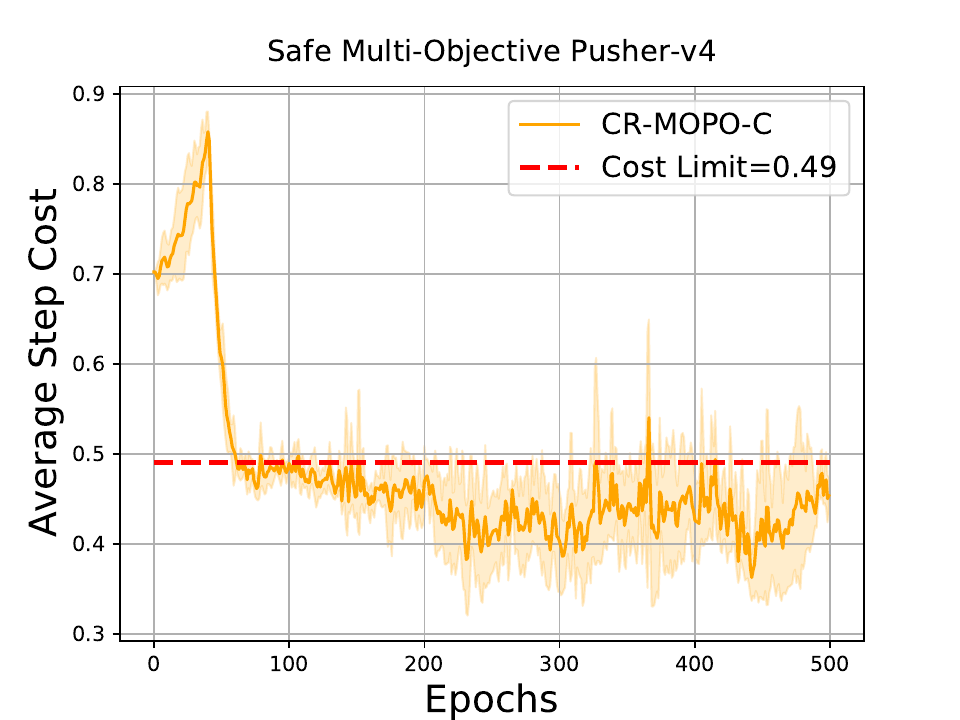}
}
    \vspace{-0pt}
 	\caption{\normalsize CR-MOPO on Safe Multi-Objective MuJoCo environments regarding the reward and safety performance.  
 	} 
  \label{fig:safe-multi-task-different-tasks}
 \end{figure} 


 \textbf{Comparison experiments with CRPO~\cite{xu2021crpo} on  Safe Multi-Objective HalfCheetah Environments.} We develop a novel algorithm, the CR-MOPO-Soft (CR-MOPO-S) constraints' algorithm, that can handle the safety constraint as one of the objectives. In CR-MOPO-S, we do not need to include the constraint in the objective. Instead, we incorporate the constraint in our method with a weight for performance (e.g., the weight can be 1.0). We hypothesize that by considering the constraint, we can navigate into a deep safe set, allowing us to concentrate on improving performance. This is especially relevant when the system is highly constrained, as operating near the boundary of the safe set can easily lead to constraint violations and "oscillations" behaviors from safe learning methods, e.g., CRPO~\cite{xu2021crpo}, CPO~\cite{achiam2017constrained}, PCPO~\cite{yang2020projection}. Intuitively, if we consistently operate at the safety boundary, we may encounter a conflict between safety and performance. Drawing an analogy to running on a road, if we position ourselves close to the edge, stepping out of the road would require effort to retreat, potentially compromising speed. However, suppose we realize the importance of running in the center of the road. In that case, we can prioritize safety and maximizing speed simultaneously. To hammer in the insight, we expand on the examples and scenarios we consider. We focus on a simple CMDP setup, prioritizing the maximization of rewards by incorporating safety violations as one of the objectives. We compare our algorithms, CR-MOPO and CR-MOPO-S, with a strong safe RL baseline, CRPO. In CRPO\cite{xu2021crpo}, the multiple tasks are scaled linearly into one objective. CRPO is a safe RL algorithm that can show better performance than safe RL algorithms, such as CPO~\cite{achiam2017constrained}, IPO~\cite{liu2020ipo}, and unsafe RL algorithms such as TRPO~\cite{schulman2015trust} in terms of the balance between reward and safety violation.
As shown in Figure~\ref{fig:cmorl-cmorl-crpos-safe-multi-task-halfcheetah-different-limit-0Dot005}, the experiment results demonstrate that our algorithms, CR-MOPO and CR-MOPO-S show better performance than CRPO regarding the reward and safety performance. And also, Our algorithms show a faster convergence rate than CRPO\cite{xu2021crpo}. The detailed implementation is provided in Appendix \ref{append:implementation-details}.

 \begin{figure}[htbp!]
 \centering
 \subcaptionbox{}
 {
  \includegraphics[width=0.246\linewidth]{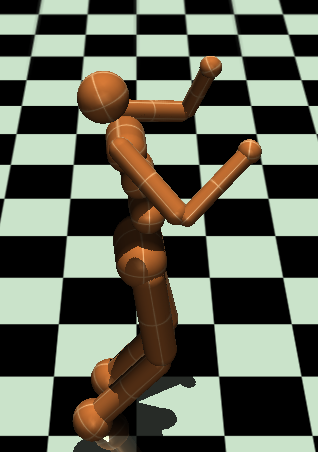}
  }
   \subcaptionbox{}
  {
\includegraphics[width=0.39\linewidth]{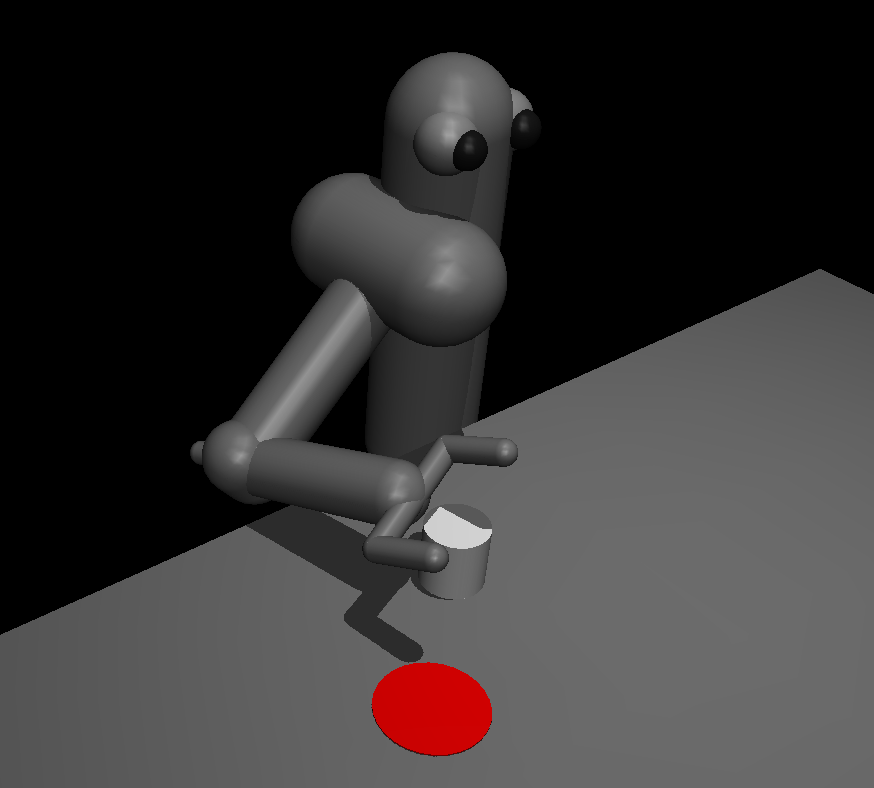}
}
 \subcaptionbox{}
 {
  \includegraphics[width=0.45\linewidth]{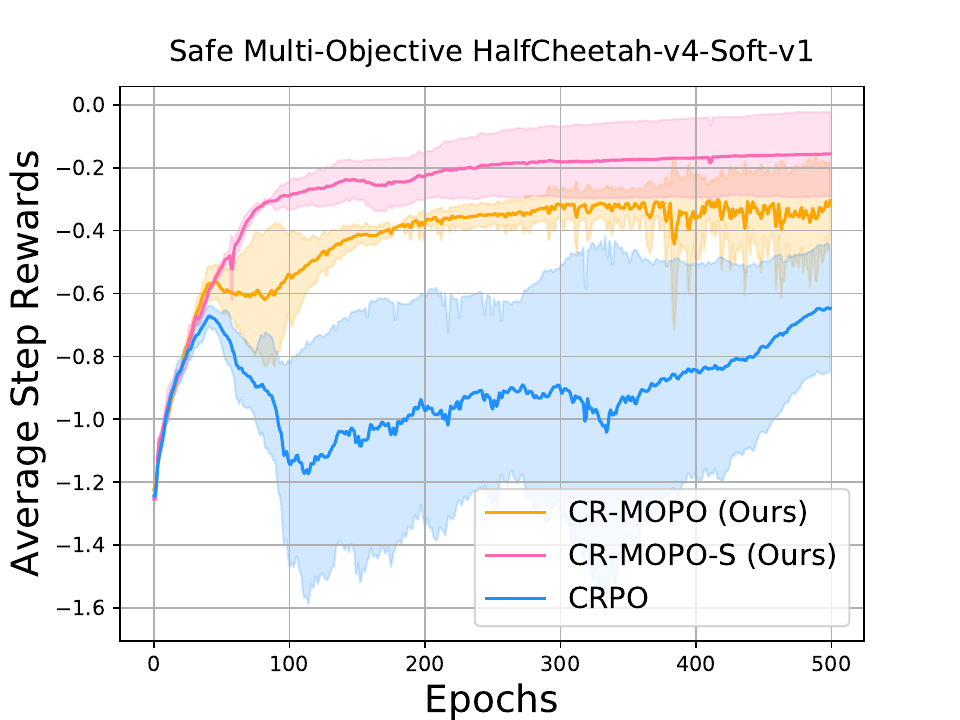}
  }
   \subcaptionbox{}
  {
\includegraphics[width=0.45\linewidth]{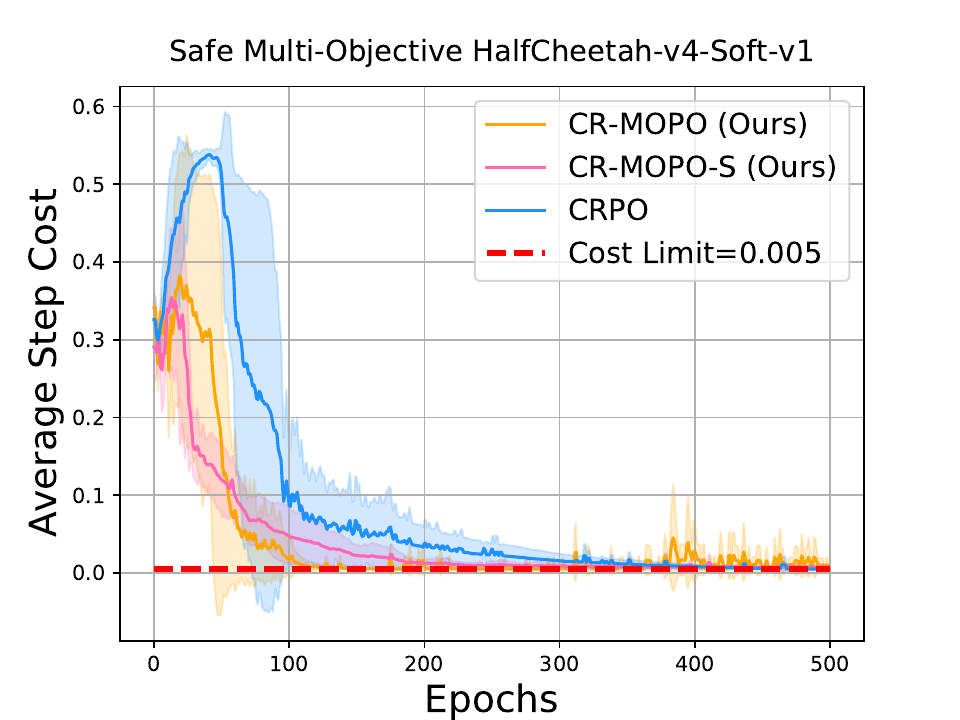}
}
    \vspace{-0pt}
 	\caption{\normalsize (a) and (b) show two of the Safe Multi-Objective MuJoCO environments, Safe Multi-Objective Humanoid and Pusher. (c) and (d) show the comparison results in terms of CR-MOPO, CR-MOPO-S and CRPO~\cite{xu2021crpo} on a Safe Multi-Objective MuJoCo environment, Safe Multi-Objective HalfCheetah, the cost limit is $0.005$, we start to optimize safety violation after $40$ Epochs. 
 	} 
  \label{fig:cmorl-cmorl-crpos-safe-multi-task-halfcheetah-different-limit-0Dot005}
 \end{figure}

\begin{figure}[htbp!]
 \centering
 \subcaptionbox{}
 {
  \includegraphics[width=0.45\linewidth]{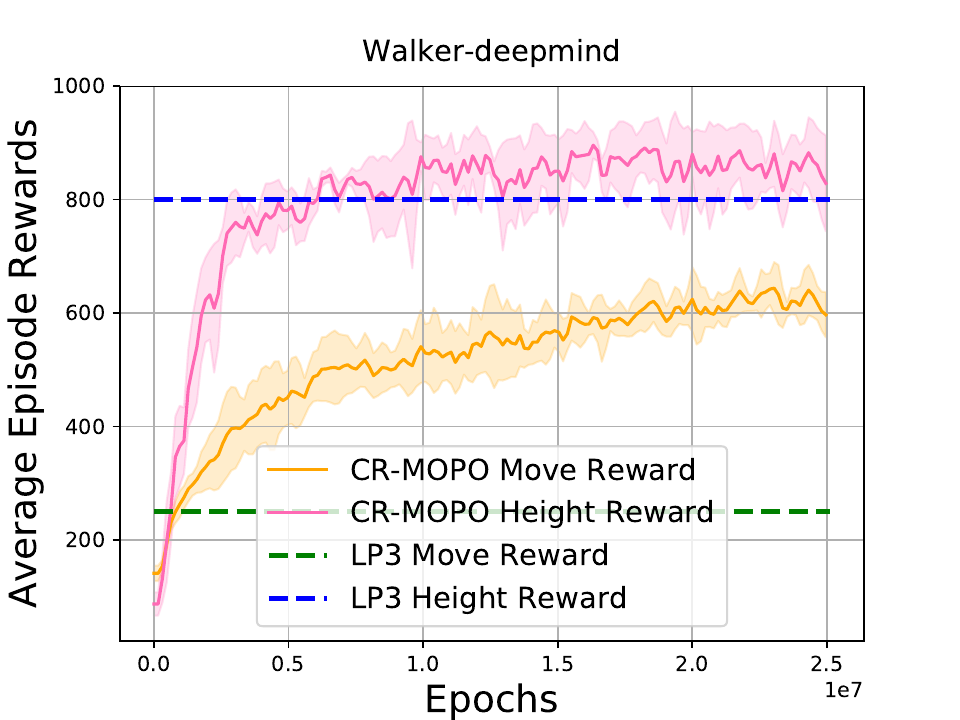}
  }
   \subcaptionbox{}
  {
\includegraphics[width=0.45\linewidth]{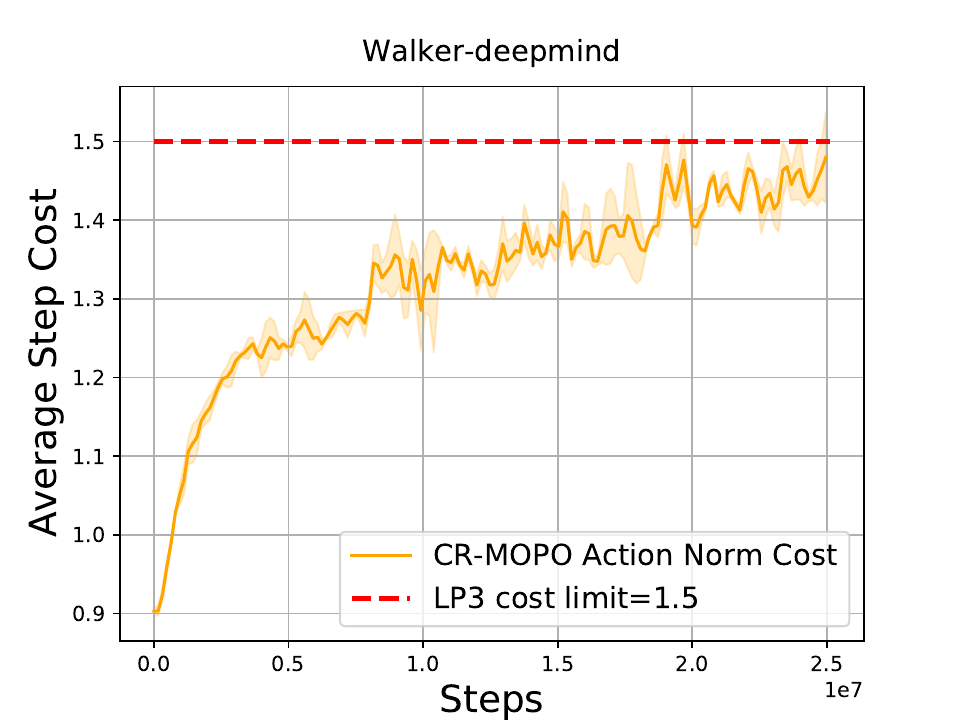}
}
\subcaptionbox{}
 {
  \includegraphics[width=0.45\linewidth]{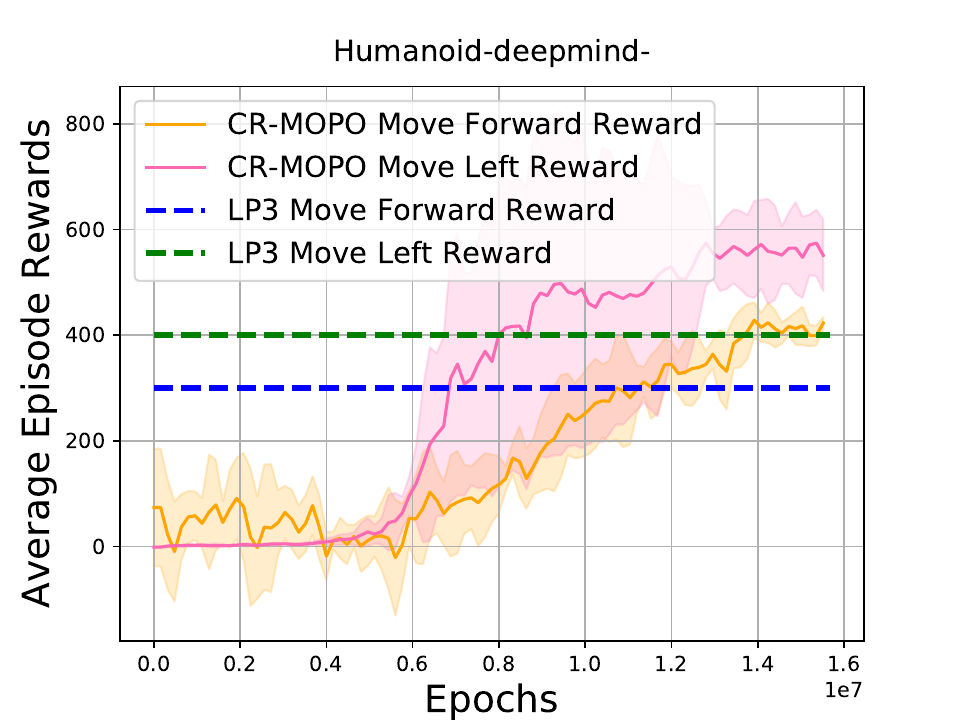}
  }
   \subcaptionbox{}
  {
\includegraphics[width=0.45\linewidth]{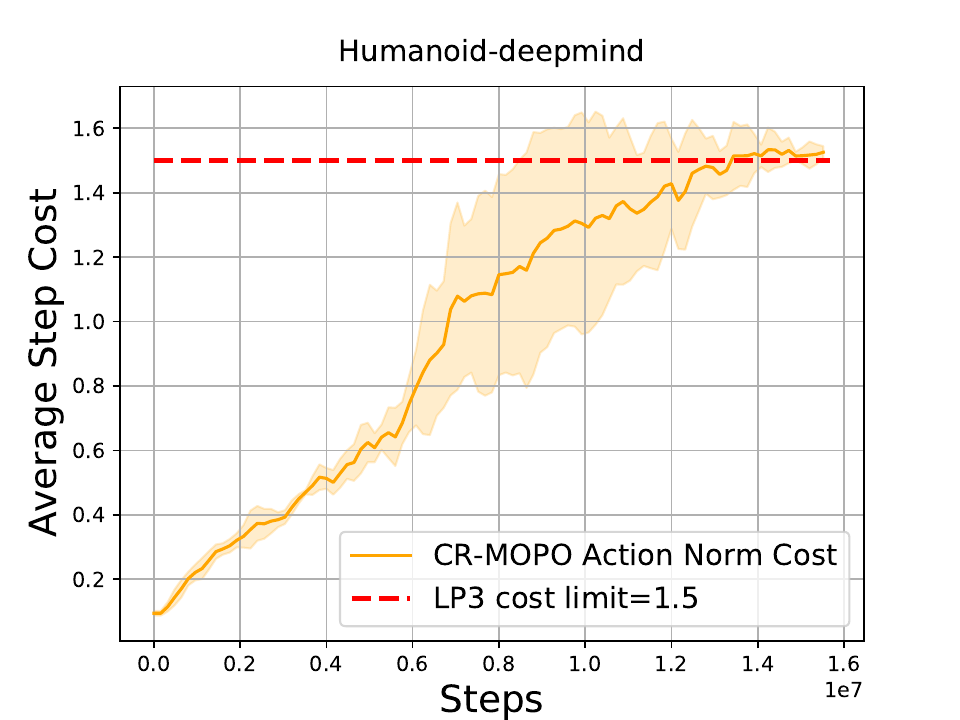}
}
    \vspace{-0pt}
 	\caption{\normalsize Compared with the DeepMind's method, LP3~\cite{huang2022constrained}, on Safe Multi-Objective Walker-dm and Safe Multi-Objective Humanoid-dm environments. 
 	} 
  \label{fig:cmorl-compared-with-deepmind-lp3-walker-humanoid}
 \end{figure}

  \textbf{Comparison experiments with LP3~\cite{huang2022constrained} on Different Safe Multi-Objective Humanoid-dm and Walker-dm.} LP3~\cite{huang2022constrained} is a SOTA-safe multi-objective RL baseline proposed by DeepMind. They achieve safe multi-objective RL by leveraging the learning preferences, and they also deploy their methods on several challenging tasks, e.g., Humanoid-dm and Walker-dm that are from the DeepMind Control Suite~\cite{tassa2018deepmind}. In this section, we compare our algorithm with LP3 on the same environment settings. As shown in Figure~\ref{fig:cmorl-compared-with-deepmind-lp3-walker-humanoid}, our algorithm can show remarkably better performance than LP3~\cite{huang2022constrained} while guaranteeing safety. As shown in Figure~\ref{fig:cmorl-compared-with-deepmind-lp3-walker-humanoid} (a) and (b), in the Walker-dm task, when the cost limit is $1.5$, and our method can perform very well, e.g., the move reward is more than 800, and the height reward is more than 600; in the same setting, using LP3~\cite{huang2022constrained}, the move reward can only be more than 250, and the height reward is about 800. For Humanoid-dm, LP3 can achieve about 400 and 300 regarding the move left and move forward reward when the cost limit is 1.5. However, our method can achieve at least 400 regarding the move left reward and about 600 regarding the move forward reward. The experiment results demonstrate that our method performs significantly better than LP3~\cite{huang2022constrained}.

  One explanation of the better results compared to LP3, could be that our methods can optimize policies to the boundaries of the constraint thresholds. Although this results in not conservative policies, this can easily be remedied by choosing more conservative constraint thresholds.
\section{Conclusion} \label{sec: conclu}

In this study, we try to balance each task's performance in a multi-task RL setting. Further, we achieve each task rewards monotonic improvement and ensure policy safety. A primal-based safe multi-task RL framework is proposed, where multiple objectives between different tasks are optimized by analyzing the conflict gradient manipulation, and the constraint rectification is leveraged to search for the safety policy during multiple objectives'  exploration. Moreover, the convergence and safety violation analysis are provided. Finally, we deploy our practical algorithms on several challenging, safe multi-task RL environments and compare our method with the SOTA safe RL baseline and safe multi-task RL algorithms. The experiment results indicate that our method can perform better than SOTA-safe RL baselines and SOTA-safe multi-task RL algorithms regarding the balance between each task performance and safety violation. It is necessary to address how to deploy our algorithm in real-world scenarios and try to leverage the foundation models~\cite{yang2023foundation} with our method to address safe multi-task RL robustness problems. A conceivable adverse societal consequence arises from the potential misapplication of this work within safety-critical contexts, which could lead to unforeseen damages. It is our aspiration that our discoveries catalyze further investigations into the safety and generalizability of reinforcement learning in secure environments.




\bibliographystyle{plain}
\bibliography{main}
\clearpage

\appendix

\begin{center}
\textbf{\Large Appendix}
\end{center}


\section{Proof} \label{sec: Proof}
We prove the results of Theorem \ref{thm: main} in two parts. We give the TD-Learning results in Theorem \ref{thm: td-learning} in section \ref{sec: td-learning}, and the unbiased q-estimation estimator results in Theorem \ref{thm: q-estimator} in section \ref{sec: Q-estimator}.

\begin{lemma}[Multi-objective NPG]\label{eq: multi-objective NPG}
Given the preference vector $\boldsymbol{\lambda}$ and considering the multi-objective NPG update $w_{t+1}=w_t- \eta \tilde{F}(w_t)^{\dagger}  \sum_{i=1}^m \lambda^i \nabla f_i\left(\pi_{w_t}\right)$, where $\tilde{F}(w)=\mathbb{E}_{\nu_{\pi_w}}\left[\phi_w(s, a) \phi_w(s, a)^{\top}\right]$ in the tabular setting, the multi-objective NPG update also takes the form:
$$
w_{t+1}=w_t+\frac{\eta}{1-\gamma} \sum_{i=1}^m \lambda^i Q_i^{\pi_{w_t}},\quad \quad \pi_{w_{t+1}}(a \mid s)=\pi_{w_t}(a \mid s) \frac{\exp \left(\eta \sum_{i=1}^m \lambda_i Q_i^{\pi_{w_t}} /(1-\gamma)\right)}{Z_t(s)} .
$$
where
$$
Z_t(s)=\sum_{a \in \mathcal{A}} \pi_{w_t}(a \mid s) \exp{ \left(\frac{\eta \sum_{i=1}^m \lambda_i  {Q}_i^{\pi_{w_t}}(s, a)}{1-\gamma}\right)}
$$
\end{lemma}

\begin{proof}
From Lemma 5.1 in \cite{agarwal2021theory}, we know
$$
\tilde{F}(w_t)^{\dagger} \nabla f_i\left(\pi_{w_t}\right)=\frac{A_i^{\pi_{w_t}}}{1-\gamma}+v
$$ 
where $v \in \mathbb{R}^{|\mathcal{S}||\mathcal{A}|}$ and
$v_{s, a}=c_s \text { for some } c_s \in \mathbb{R} \text { for each state } s$ and action $a$. This yields the updates
$$
w_{t+1}=w_t+\frac{\eta}{1-\gamma} \sum_{i=1}^m \lambda^i A_i^{\pi_{w_t}}+\eta v\sum_{i=1}^m \lambda^i\quad$$
and
$$\pi_{w_{t+1}}(a \mid s)=\pi_{w_t}(a \mid s) \frac{\exp \left(\eta \sum_{i=1}^m \lambda_i A_i^{\pi_{w_t}} /(1-\gamma)+\eta c_s\sum_{i=1}^m \lambda_i \right)}{Z_t(s)} .
$$
Owing to the normalization factor $Z_t(s)$, the state-dependent offset $c_s$ cancels in the updates for $\pi$, so that resulting policy is invariant to the specific choice of $c_s$. Hence, we pick $c_s \equiv 0$, which yields the updates
$$
w_{t+1}=w_t+\frac{\eta}{1-\gamma} \sum_{i=1}^m \lambda^i A_i^{\pi_{w_t}}$$
and
$$\pi_{w_{t+1}}(a \mid s)=\pi_{w_t}(a \mid s) \frac{\exp \left(\eta \sum_{i=1}^m \lambda^i A_i^{\pi_{w_t}} /(1-\gamma)\right)}{Z_t(s)} 
$$
Finally, the advantage function $A_i^{\pi_w}(s,a)=Q_i^{\pi_w}(s,a)-V_i^{\pi_w}(s,a) $ can be replaced by the Q-function  $Q_i^{\pi_w}(s,a) $ due to the normalization factor $Z_t(s)$, which yields the the statement of the lemma.

\end{proof}

\begin{lemma}
\label{lemma: Performance improvement bound for approximated multi-objective NPG}

\textbf{Performance improvement bound for approximated multi-objective NPG}.
 For the iterates $\pi_{w_t}$ generated by the approximated multi-objective NPG updates in the tabular setting, we have the following holds
$$
\begin{aligned}
&\sum_{i=1}^m \lambda^i \left(f_i\left(w_{t+1}\right)-f_i\left(w_t\right)\right)\\
{\geq} &\frac{1-\gamma}{\eta} \mathbb{E}_{s \sim \rho}\left(\log Z_t(s)-\frac{\eta}{1-\gamma}  \sum_{i=1}^m \lambda^iV^{\pi_{w_t}}_i(s)+\frac{\eta}{1-\gamma} \sum_{a \in \mathcal{A}} \pi_{w_t}(a \mid s) \sum_{i=1}^m \lambda^i\left(Q^{\pi_{w_t}}_i(s, a)-\bar{Q}^t_i(s, a)\right)\right) \\
& -\frac{1}{1-\gamma} \mathbb{E}_{s \sim \nu_\rho} \sum_{a \in \mathcal{A}} \pi_{w_t}(a \mid s) \sum_{i=1}^m \lambda^i\left(Q^{\pi_{w_t}}_i(s, a)-\bar{Q}^t_i(s, a)\right) \\
& +\frac{1}{1-\gamma} \mathbb{E}_{s \sim \nu_\rho} \sum_{a \in \mathcal{A}} \pi_{w_{t+1}}(a \mid s) \sum_{i=1}^m \lambda^i\left(Q^{\pi_{w_t}}_i(s, a)-\bar{Q}^t_i(s, a)\right)  .
\end{aligned}
$$    
\end{lemma}

\begin{proof}
 We first provide the following lower bound.
$$
\begin{aligned}
& \log Z_t(s)-\frac{\eta}{1-\gamma} \sum_{i=1}^m \lambda^iV^{\pi_{w_t}}_i(s) \\
& =\log \sum_{a \in \mathcal{A}} \pi_{w_t}(a \mid s) \exp \left(\frac{\eta \sum_{i=1}^m \lambda^i \bar{Q}^t_i(s, a)}{1-\gamma}\right)-\frac{\eta}{1-\gamma} \sum_{i=1}^m \lambda^iV^{\pi_{w_t}}_i(s) \\
& \geq \sum_{a \in \mathcal{A}} \pi_{w_t}(a \mid s) \log \exp \left(\frac{\eta \sum_{i=1}^m \lambda^i \bar{Q}^t_i(s, a)}{1-\gamma}\right)-\frac{\eta}{1-\gamma} \sum_{i=1}^m \lambda^iV^{\pi_{w_t}}_i(s) \\
& =\frac{\eta}{1-\gamma} \sum_{a \in \mathcal{A}} \pi_{w_t}(a \mid s)\sum_{i=1}^m \lambda^i\left(\bar{Q}^t_i(s, a)-Q^{\pi_{w_t}}_i(s, a)\right)\\
&+\frac{\eta}{1-\gamma} \sum_{a \in \mathcal{A}} \pi_{w_t}(a \mid s) \sum_{i=1}^m \lambda^i Q^{\pi_{w_t}}_i(s, a)-\frac{\eta}{1-\gamma} \sum_{i=1}^m \lambda^iV^{\pi_{w_t}}_i(s) \\
& =\frac{\eta}{1-\gamma} \sum_{a \in \mathcal{A}} \pi_{w_t}(a \mid s)\sum_{i=1}^m \lambda^i \left(\bar{Q}^t_i(s, a)-Q^{\pi_{w_t}}_i(s, a)\right) 
\end{aligned}
$$
Thus, we conclude that
\begin{equation} \label{eq: non-negative term in lemma}
\log Z_t(s)-\frac{\eta}{1-\gamma} \sum_{i=1}^m \lambda^i V^{\pi_{w_t}}_i(s)+\frac{\eta}{1-\gamma} \sum_{a \in \mathcal{A}} \pi_{w_t}(a \mid s)\sum_{i=1}^m \lambda^i\left(Q^{\pi_{w_t}}_i(s, a)-\bar{Q}^t_i(s, a)\right) \geq 0 .
\end{equation}
We then proceed to prove this lemma. The performance difference lemma \cite{kakade2001natural} implies:
\begin{align*}
& \sum_{i=1}^m \lambda^i\left(f_i\left(w_{t+1}\right)-f_i\left(w_t\right)\right) \\
&= \frac{1}{1-\gamma} \mathbb{E}_{s \sim \nu_\rho} \sum_{a \in \mathcal{A}} \pi_{w_{t+1}}(a \mid s) \sum_{i=1}^m \lambda^i A^{\pi_{w_t}}_i(s, a) \\
&= \frac{1}{1-\gamma} \mathbb{E}_{s \sim \nu_\rho} \sum_{a \in \mathcal{A}} \pi_{w_{t+1}}(a \mid s) \sum_{i=1}^m \lambda^iQ^{\pi_{w_t}}_i(s, a)-\frac{1}{1-\gamma} \mathbb{E}_{s \sim \nu_\rho}\sum_{i=1}^m \lambda^i V^{\pi_{w_t}}_i(s) \\
&= \frac{1}{1-\gamma} \mathbb{E}_{s \sim \nu_\rho} \sum_{a \in \mathcal{A}} \pi_{w_{t+1}}(a \mid s) \sum_{i=1}^m \lambda^i\bar{Q}^t_i(s, a)\\
&+\frac{1}{1-\gamma} \mathbb{E}_{s \sim \nu_\rho} \sum_{a \in \mathcal{A}} \pi_{w_{t+1}}(a \mid s)\sum_{i=1}^m \lambda^i\left(Q^{\pi_{w_t}}_i(s, a)-\bar{Q}^t_i(s, a)\right) -\frac{1}{1-\gamma} \mathbb{E}_{s \sim \nu_\rho}\sum_{i=1}^m \lambda^i V^{\pi_{w_t}}_i(s) \\
& \stackrel{(i)}{=} \frac{1}{\eta} \mathbb{E}_{s \sim \nu_\rho} \sum_{a \in \mathcal{A}} \pi_{w_{t+1}}(a \mid s) \log \left(\frac{\pi_{w_{t+1}}(a \mid s) Z_t(s)}{\pi_{w_t}(a \mid s)}\right) \\
&+\frac{1}{1-\gamma} \mathbb{E}_{s \sim \nu_\rho} \sum_{a \in \mathcal{A}} \pi_{w_{t+1}}(a \mid s)\sum_{i=1}^m \lambda^i\left(Q^{\pi_{w_t}}_i(s, a)-\bar{Q}^t_i(s, a)\right)-\frac{1}{1-\gamma} \mathbb{E}_{s \sim \nu_\rho} \sum_{i=1}^m \lambda^iV^{\pi_{w_t}}_i(s) 
\end{align*}
\begin{align*}
&= \frac{1}{\eta} \mathbb{E}_{s \sim \nu_\rho} D_{\mathrm{KL}}\left(\pi_{w_{t+1}}|| \pi_{w_t}\right)+\frac{1}{\eta} \mathbb{E}_{s \sim \nu_\rho} \log Z_t(s) \\
&+\frac{1}{1-\gamma} \mathbb{E}_{s \sim \nu_\rho} \sum_{a \in \mathcal{A}} \pi_{w_{t+1}}(a \mid s)\sum_{i=1}^m \lambda^i\left(Q^{\pi_{w_t}}_i(s, a)-\bar{Q}^t_i(s, a)\right)-\frac{1}{1-\gamma} \mathbb{E}_{s \sim \nu_\rho} \sum_{i=1}^m \lambda^iV^{\pi_{w_t}}_i(s) \\
& \geq \frac{1}{\eta} \mathbb{E}_{s \sim \nu_\rho}\left(\log Z_t(s)-\frac{\eta}{1-\gamma} \sum_{i=1}^m \lambda^i V^{\pi_{w_t}}_i(s)+\frac{\eta}{1-\gamma} \sum_{a \in \mathcal{A}} \pi_{w_t}(a \mid s)\sum_{i=1}^m \lambda^i\left(Q^{\pi_{w_t}}_i(s, a)-\bar{Q}^t_i(s, a)\right)\right)\\
& -\frac{1}{1-\gamma} \mathbb{E}_{s \sim \nu_\rho} \sum_{a \in \mathcal{A}} \pi_{w_t}(a \mid s) \sum_{i=1}^m \lambda^i\left(Q^{\pi_{w_t}}_i(s, a)-\bar{Q}^t_i(s, a)\right) \\
& +\frac{1}{1-\gamma} \mathbb{E}_{s \sim \nu_\rho} \sum_{a \in \mathcal{A}} \pi_{w_{t+1}}(a \mid s) \sum_{i=1}^m \lambda^i\left(Q^{\pi_{w_t}}_i(s, a)-\bar{Q}^t_i(s, a)\right) \\
& \stackrel{(ii)}{\geq} \frac{1-\gamma}{\eta} \mathbb{E}_{s \sim \rho}\left(\log Z_t(s)-\frac{\eta}{1-\gamma}  \sum_{i=1}^m \lambda^iV^{\pi_{w_t}}_i(s)+\frac{\eta}{1-\gamma} \sum_{a \in \mathcal{A}} \pi_{w_t}(a \mid s) \sum_{i=1}^m \lambda^i\left(Q^{\pi_{w_t}}_i(s, a)-\bar{Q}^t_i(s, a)\right)\right) \\
& -\frac{1}{1-\gamma} \mathbb{E}_{s \sim \nu_\rho} \sum_{a \in \mathcal{A}} \pi_{w_t}(a \mid s) \sum_{i=1}^m \lambda^i\left(Q^{\pi_{w_t}}_i(s, a)-\bar{Q}^t_i(s, a)\right) \\
& +\frac{1}{1-\gamma} \mathbb{E}_{s \sim \nu_\rho} \sum_{a \in \mathcal{A}} \pi_{w_{t+1}}(a \mid s) \sum_{i=1}^m \lambda^i\left(Q^{\pi_{w_t}}_i(s, a)-\bar{Q}^t_i(s, a)\right) 
\end{align*}
where $(i)$ follows from the update rule in Lemma \ref{eq: multi-objective NPG} and $(ii)$ follows from the facts that $\norm{\nu_\rho/\rho}\geq 1-\gamma$ and \eqref{eq: non-negative term in lemma}.
\end{proof}

\begin{lemma}\label{lemma: optimality gap for approximated NPG}
\textbf{Expected optimality gap for approximated multi-objective NPG}.
Consider the approximated NPG updates in the tabular setting. We have
$$
\begin{aligned}
& \sum_{i=1}^m \lambda^i \left(f_i\left(\pi^*\right)-f_i\left(\pi_{w_t}\right) \right) \\
&{\leq}  \frac{1}{\eta} \mathbb{E}_{s \sim \nu^*}\left(D_{\mathrm{KL}}\left(\pi^* \| \pi_{w_t}\right)-D_{\mathrm{KL}}\left(\pi^* \| \pi_{w_{t+1}}\right)\right)  +\frac{2 \eta r_{\max} }{(1-\gamma)^2}  \sum_{i=1}^m \lambda^i \left\|\sum_{i=1}^m \lambda^i\bar{Q}_t^i\right\|_2\\
& +\frac{1}{(1-\gamma)^2} \mathbb{E}_{s \sim \nu_\rho} \sum_{a \in \mathcal{A}} \pi_{w_t}(a \mid s) \sum_{i=1}^m \lambda^i\left(Q^{\pi_{w_t}}_i(s, a)-\bar{Q}^t_i(s, a)\right) \\
& -\frac{1}{(1-\gamma)^2} \mathbb{E}_{s \sim \nu_\rho} \sum_{a \in \mathcal{A}} \pi_{w_{t+1}}(a \mid s) \sum_{i=1}^m \lambda^i\left(Q^{\pi_{w_t}}_i(s, a)-\bar{Q}^t_i(s, a)\right)\\
& +\frac{1}{1-\gamma} \mathbb{E}_{s \sim \nu^*} \sum_{a \in \mathcal{A}} \pi^*(a \mid s)\sum_{i=1}^m  \lambda^i\left(Q^{\pi^{w_t}}_i(s, a)-\bar{Q}^t_i(s, a)\right)
\end{aligned}
$$
\end{lemma}

\begin{proof}
\begin{figure*}

\begin{scriptsize}
    \centering  
\begin{align*}
\scriptsize
& \sum_{i=1}^m \lambda^i \left(f_i\left(\pi^*\right)-f_i\left(\pi_{w_t}\right) \right) \\
& =\frac{1}{1-\gamma} \mathbb{E}_{s \sim \nu^*} \sum_{a \in \mathcal{A}} \pi^*(a \mid s) \sum_{i=1}^m \lambda^i A^{\pi_{w_t}}_i(s, a) \\
& =\frac{1}{1-\gamma} \mathbb{E}_{s \sim \nu^*} \sum_{a \in \mathcal{A}} \pi^*(a \mid s) \sum_{i=1}^m \lambda^i Q^{\pi_{w_t}}_i(s, a)-\frac{1}{1-\gamma} \mathbb{E}_{s \sim \nu^*} \sum_{i=1}^m \lambda^i V^{\pi_{w_t}}_i(s)\\
&= \frac{1}{1-\gamma} \mathbb{E}_{s \sim \nu^*} \sum_{a \in \mathcal{A}} \pi^*(a \mid s) \sum_{i=1}^m  \lambda^i \bar{Q}^t_i(s, a)+\frac{1}{1-\gamma} \mathbb{E}_{s \sim \nu^*} \sum_{a \in \mathcal{A}} \pi^*(a \mid s) \sum_{i=1}^m  \lambda^i\left(Q^{\pi_{w_t}}_i(s, a)-\bar{Q}^t_i(s, a)\right) \\
&-\frac{1}{1-\gamma} \mathbb{E}_{s \sim \nu^*} \sum_{i=1}^m \lambda^i V^{\pi_{w_t}}_i(s) \\
& \stackrel{(i)}{=} \frac{1}{\eta} \mathbb{E}_{s \sim \nu^*} \sum_{a \in \mathcal{A}} \pi^*(a \mid s) \log \frac{\pi_{w_{t+1}}(a \mid s) Z_t(s)}{\pi_{w_t}(a \mid s)}+\frac{1}{1-\gamma} \mathbb{E}_{s \sim \nu^*} \sum_{a \in \mathcal{A}} \pi^*(a \mid s) \sum_{i=1}^m  \lambda^i \left(Q^{\pi_{w_t}}_i(s, a)-\bar{Q}^t_i(s, a)\right) \\
&-\frac{1}{1-\gamma} \mathbb{E}_{s \sim \nu^*} \sum_{i=1}^m \lambda^i V^{\pi_{w_t}}_i(s) \\
& =\frac{1}{\eta} \mathbb{E}_{s \sim \nu^*}\left(D_{\mathrm{KL}}\left(\pi^*|| \pi_{w_t}\right)-D_{\mathrm{KL}}\left(\pi^*|| \pi_{w_{t+1}}\right)\right)+\frac{1}{1-\gamma} \mathbb{E}_{s \sim \nu^*} \sum_{a \in \mathcal{A}} \pi^*(a \mid s) \sum_{i=1}^m  \lambda^i\left(Q^{\pi_{w_t}}_i(s, a)-\bar{Q}^t_i(s, a)\right) \\
&+\frac{1}{\eta} \mathbb{E}_{s \sim \nu^*}\left(\log Z_t(s)-\frac{\eta}{1-\gamma} \sum_{i=1}^m \lambda^i V^{\pi_{w_t}}_i(s)\right)  \\
& \leq \frac{1}{\eta} \mathbb{E}_{s \sim \nu^*}\left(D_{\mathrm{KL}}\left(\pi^*|| \pi_{w_t}\right)-D_{\mathrm{KL}}\left(\pi^*|| \pi_{w_{t+1}}\right)\right) \\
&+\frac{1}{\eta} \mathbb{E}_{s \sim \nu^*}\left(\log Z_t(s)-\frac{\eta}{1-\gamma} \sum_{i=1}^m \lambda^i V^{\pi_{w_t}}_i(s)+\frac{\eta}{1-\gamma} \sum_{a \in \mathcal{A}} \pi_{w_t}(a \mid s)\sum_{i=1}^m \lambda^i \left(Q^{\pi_{w_t}}_i(s, a)-\bar{Q}^t_i(s, a)\right)\right)\\
& +\frac{1}{1-\gamma} \mathbb{E}_{s \sim \nu^*} \sum_{a \in \mathcal{A}} \pi^*(a \mid s) \sum_{i=1}^m  \lambda^i \left(Q^{\pi_{w_t}}_i(s, a)-\bar{Q}^t_i(s, a)\right)\\
&\stackrel{(i i)}{\leq}  \frac{1}{\eta} \mathbb{E}_{s \sim \nu^*}\left(D_{\mathrm{KL}}\left(\pi^* \| \pi_{w_t}\right)-D_{\mathrm{KL}}\left(\pi^* \| \pi_{w_{t+1}}\right)\right)  +\frac{1}{1-\gamma} \sum_{i=1}^m \lambda^i \left(f_i^{\nu^*}\left(w_{t+1}\right)-f_i^{\nu^*}\left(w_t\right)\right)\\
& +\frac{1}{(1-\gamma)^2} \mathbb{E}_{s \sim \nu_\rho} \sum_{a \in \mathcal{A}} \pi_{w_t}(a \mid s) \sum_{i=1}^m \lambda^i\left(Q^{\pi_{w_t}}_i(s, a)-\bar{Q}^t_i(s, a)\right) \\
& -\frac{1}{(1-\gamma)^2} \mathbb{E}_{s \sim \nu_\rho} \sum_{a \in \mathcal{A}} \pi_{w_{t+1}}(a \mid s) \sum_{i=1}^m \lambda^i\left(Q^{\pi_{w_t}}_i(s, a)-\bar{Q}^t_i(s, a)\right)\\
& +\frac{1}{1-\gamma} \mathbb{E}_{s \sim \nu^*} \sum_{a \in \mathcal{A}} \pi^*(a \mid s)\sum_{i=1}^m  \lambda^i\left(Q^{\pi^{w_t}}_i(s, a)-\bar{Q}^t_i(s, a)\right) \\
&\stackrel{(iii)}{\leq}  \frac{1}{\eta} \mathbb{E}_{s \sim \nu^*}\left(D_{\mathrm{KL}}\left(\pi^* \| \pi_{w_t}\right)-D_{\mathrm{KL}}\left(\pi^* \| \pi_{w_{t+1}}\right)\right)  +\frac{2 r_{\max} }{(1-\gamma)^2}  \sum_{i=1}^m \lambda^i \left\|w_{t+1}-w_t\right\|_2\\
& +\frac{1}{(1-\gamma)^2} \mathbb{E}_{s \sim \nu_\rho} \sum_{a \in \mathcal{A}} \pi_{w_t}(a \mid s) \sum_{i=1}^m \lambda^i\left(Q^{\pi_{w_t}}_i(s, a)-\bar{Q}^t_i(s, a)\right) \\
& -\frac{1}{(1-\gamma)^2} \mathbb{E}_{s \sim \nu_\rho} \sum_{a \in \mathcal{A}} \pi_{w_{t+1}}(a \mid s) \sum_{i=1}^m \lambda^i\left(Q^{\pi_{w_t}}_i(s, a)-\bar{Q}^t_i(s, a)\right)\\
& +\frac{1}{1-\gamma} \mathbb{E}_{s \sim \nu^*} \sum_{a \in \mathcal{A}} \pi^*(a \mid s)\sum_{i=1}^m  \lambda^i\left(Q^{\pi^{w_t}}_i(s, a)-\bar{Q}^t_i(s, a)\right) \\
&\stackrel{(iiii)}{\leq}  \frac{1}{\eta} \mathbb{E}_{s \sim \nu^*}\left(D_{\mathrm{KL}}\left(\pi^* \| \pi_{w_t}\right)-D_{\mathrm{KL}}\left(\pi^* \| \pi_{w_{t+1}}\right)\right)  +\frac{2 \eta r_{\max} }{(1-\gamma)^2}  \sum_{i=1}^m \lambda^i \left\|\sum_{i=1}^m \lambda^i\bar{Q}_t^i\right\|_2\\
& +\frac{1}{(1-\gamma)^2} \mathbb{E}_{s \sim \nu_\rho} \sum_{a \in \mathcal{A}} \pi_{w_t}(a \mid s) \sum_{i=1}^m \lambda^i\left(Q^{\pi_{w_t}}_i(s, a)-\bar{Q}^t_i(s, a)\right) \\
& -\frac{1}{(1-\gamma)^2} \mathbb{E}_{s \sim \nu_\rho} \sum_{a \in \mathcal{A}} \pi_{w_{t+1}}(a \mid s) \sum_{i=1}^m \lambda^i\left(Q^{\pi_{w_t}}_i(s, a)-\bar{Q}^t_i(s, a)\right)\\
& +\frac{1}{1-\gamma} \mathbb{E}_{s \sim \nu^*} \sum_{a \in \mathcal{A}} \pi^*(a \mid s)\sum_{i=1}^m  \lambda^i\left(Q^{\pi^{w_t}}_i(s, a)-\bar{Q}^t_i(s, a)\right)
\end{align*}
\end{scriptsize}
\end{figure*}
where $(i)$ and $(iiii)$ follow from Lemma \ref{eq: multi-objective NPG}, $(ii)$ follows from Lemma \ref{lemma: Performance improvement bound for approximated multi-objective NPG}, and $(iii)$ follows from the  Lipschitz property of $f_i^{\nu^*}\left(w\right)$ such that $f_i^{\nu^*}\left(w_{t+1}\right)-f_i^{\nu^*}\left(w_t\right)\leq \frac{2c_{\text{max}}}{1-\gamma} \norm{w_{t+1} -w_t}_2$.
\end{proof}

\subsection{TD-Learning}\label{sec: td-learning}
\begin{lemma}[\cite{dalal2018finite}] \label{lemma: TD learning convergence}
Consider the iteration given in \eqref{eq: TD learning} with arbitrary initialization $Q_0^i$. Assume that the stationary distribution $\mu_{\pi_w}$ is not degenerate for all $w \in \mathbb{R}^{|\mathcal{S}| \times|\mathcal{A}|}$. Let stepsize  $\ell_k=\Theta\left(\frac{1}{k^\sigma}\right)(0<\sigma<1)$. Then, with probability at least $1-\delta$, we have
$$
\left\|Q_K^i-Q_{\pi_w}^i \right\|_2=\mathcal{O}\left(\frac{\log \left(|\mathcal{S}|^2|\mathcal{A}|^2 K^2 / \delta\right)}{(1-\gamma) K^{\sigma / 2}}\right).
$$
Note that $\sigma$ can be arbitrarily close to 1 . Lemma 2 implies that we can obtain an approximation $\bar{Q}_t^i$ such that $\left\|\bar{Q}_t^i-Q_{\pi_w}^i\right\|_2=\widetilde{\mathcal{O}}\left(1 / \sqrt{K_{\text{TD}}}\right)$ with high probability.
\end{lemma}

\begin{lemma} \label{lemma: condition for convergence td}
If 
\begin{align} \label{eq: condition for convergence td}
\beta > & \frac{2}{\eta T}\mathbb{E}_{s \sim \nu^*} D_{\mathrm{KL}}\left(\pi^*|| \pi_{w_0}\right)+ \frac{4\eta r_{\max}^2|\mathcal{S}||\mathcal{A}|}{(1-\gamma)^4 T} \left(\left[\sum_{t \in \mathcal{N}_0} \left(\sum_{i=1}^m \lambda^i_t \right)^2 \right] +T-\left|\mathcal{N}_0 \right| \right)\\
\nonumber &+\frac{6 }{(1-\gamma)^2 T} \left( \sum_{t \in \mathcal{N}_0} \sum_{i=1}^m \lambda_t^i \left\| Q_{\pi_{w_t}}^i-\bar{Q}_t^i \right\|_2+\sum_{i=m+1}^{m+p} \sum_{t \in \mathcal{N}_i}\left\|Q_{\pi_{w_t}^i}^i-\bar{Q}_t^i\right\|_2\right),
\end{align}
then we have the following holds:
\begin{enumerate}
    \item $\mathcal{N}_0 \neq \emptyset$,
    \item One of the following two statements must hold,
    \begin{enumerate}
        \item $ \left| \mathcal{N}_0 \right| \geq \frac{T}{2}$,
        \item $ \sum_{t \in \mathcal{N}_0}\sum_{i=1}^m \lambda_t^i \left(f_i\left(\pi^*\right)-f_i\left(\pi_{w_t}\right) \right)\leq 0$.
    \end{enumerate}
\end{enumerate}
\end{lemma}

\begin{proof}
When $t\in \mathcal{N}_0$, from Lemma \ref{lemma: optimality gap for approximated NPG} we have
\begin{scriptsize}
\begin{align}\label{eq: t in N0 td}
\nonumber \sum_{i=1}^m \lambda^i \left(f_i\left(\pi^*\right)-f_i\left(\pi_{w_t}\right) \right) &\leq \frac{1}{\eta} \mathbb{E}_{s \sim \nu^*}\left(D_{\mathrm{KL}}\left(\pi^*|| \pi_{w_t}\right)-D_{\mathrm{KL}}\left(\pi^*|| \pi_{w_{t+1}}\right)\right)+\frac{2 \eta c_{\max }^2|\mathcal{S}||\mathcal{A}|}{(1-\gamma)^4}(\sum_{i=1}^m \lambda^i)^2 \\
&+\frac{3}{(1-\gamma)^2}  \sum_{i=1}^m \lambda^i \left\|Q_{\pi_{w_t}}^i-\bar{Q}_t^i\right\|_2.
\end{align}
\end{scriptsize}
Similarly, when $t\in \mathcal{N}_i (i\neq 0)$ we have
\begin{align}\label{eq: t in Ni td}
    \nonumber f_i\left(\pi_{w_t}\right)-f_i\left(\pi^*\right)  
    \nonumber & \leq \frac{1}{\eta}\mathbb{E}_{s \sim \nu^*}\left(D_{\mathrm{KL}}\left(\pi^*|| \pi_{w_t}\right)-D_{\mathrm{KL}}\left(\pi^*|| \pi_{w_{t+1}}\right)\right)+\frac{2 \eta c_{\max }^2|\mathcal{S}||\mathcal{A}|}{(1-\gamma)^4}\\
    &+\frac{3}{(1-\gamma)^2} \left\|  Q_{\pi_{w_t}}^i-\bar{Q}_t^i\right\|_2.
\end{align}

Taking the summation of \eqref{eq: t in N0 td} and \eqref{eq: t in Ni td} from $t=0$ to $T-1$ gives
\begin{align} \label{eq: summation bound 1 td}
&\sum_{t \in \mathcal{N}_0}\sum_{i=1}^m \lambda_t^i \left(f_i\left(\pi^*\right)-f_i\left(\pi_{w_t}\right) \right)+\sum_{i=1}^p \sum_{t \in \mathcal{N}_i}\left(f_i\left(\pi_{w_t}\right)-f_i\left(\pi^*\right) \right) \\
\nonumber & \leq \frac{1}{\eta}\mathbb{E}_{s \sim \nu^*} D_{\mathrm{KL}}\left(\pi^*|| \pi_{w_0}\right)+\frac{2 \eta c_{\max }^2|\mathcal{S}||\mathcal{A}|}{(1-\gamma)^4} \left(\sum_{t \in \mathcal{N}_0} \left(\sum_{i=1}^m \lambda^i_t \right)^2 + T - \left|\mathcal{N}_0 \right| \right)\\
\nonumber &+\frac{3}{(1-\gamma)^2} \left( \sum_{t \in \mathcal{N}_0}  \sum_{i=1}^m \lambda^i \left\|Q_{\pi_{w_t}}^i-\bar{Q}_t^i\right\|_2+\sum_{i=m+1}^{m+p} \sum_{t \in \mathcal{N}_i}\left\|Q_{\pi_{w_t}^i}^i-\bar{Q}_t^i\right\|_2\right) .
\end{align}
Note that when $t \in \mathcal{N}_i(i \neq 0)$, we have $\bar{f}_i\left(w_t\right)>c_i+\beta$ (line 9 in Algorithm 1), which implies that
\begin{align} \label{eq: summation bound 2 td}
\nonumber f_i\left(\pi_{w_t}\right)-f_i\left(\pi^*\right) & \geq \bar{f}_i\left(w_t^i\right)-f_i\left(\pi^*\right)-\left|\bar{f}_i\left(w_t^i\right)-f_i\left(\pi_{w_t}\right)\right| \\
\nonumber & \geq c_i+\beta-J_i\left(\pi^*\right)-\left|\bar{f}_i\left(w_t^i\right)-f_i\left(\pi_{w_t}\right)\right| \\
& \geq \beta - \left\|Q_{\pi_w}^i-\bar{Q}_t^i\right\|_2.
\end{align}
Substituting \eqref{eq: summation bound 2 td} into \eqref{eq: summation bound 1 td} gives
\begin{align}
\nonumber & \sum_{t \in \mathcal{N}_0}\sum_{i=1}^m \lambda_t^i \left(f_i\left(\pi^*\right)-f_i\left(\pi_{w_t}\right) \right)+ \beta \sum_{i=1}^p\left|\mathcal{N}_i\right| - \sum_{i=m+1}^{m+p} \sum_{t \in \mathcal{N}_i}\left\|Q_{\pi_w}^i-\bar{Q}_t^i\right\|_2  \\
\nonumber \leq &\frac{1}{\eta}\mathbb{E}_{s \sim \nu^*} D_{\mathrm{KL}}\left(\pi^*|| \pi_{w_0}\right)+\frac{2 \eta c_{\max }^2|\mathcal{S}||\mathcal{A}|}{(1-\gamma)^4} \left(\sum_{t \in \mathcal{N}_0} \left(\sum_{i=1}^m \lambda^i_t\right)^2 +T-\left|\mathcal{N}_0 \right| \right)\\
\nonumber &+\frac{3}{(1-\gamma)^2} \left(\sum_{t \in \mathcal{N}_0} \sum_{i=1}^m \lambda^i  \left\|Q_{\pi_{w_t}}^i-\bar{Q}_t^i\right\|_2+\sum_{i=m+1}^{m+p} \sum_{t \in \mathcal{N}_i}\left\|Q_{\pi_{w_t}^i}^i-\bar{Q}_t^i\right\|_2\right),
\end{align}
which implies 
\begin{align} \label{eq: summation upper bound td}
& \sum_{t \in \mathcal{N}_0}\sum_{i=1}^m \lambda_t^i \left(f_i\left(\pi^*\right)-f_i\left(\pi_{w_t}\right) \right)+ \beta \sum_{i=1}^p\left|\mathcal{N}_i\right|  \\
\nonumber \leq &\frac{1}{\eta}\mathbb{E}_{s \sim \nu^*} D_{\mathrm{KL}}\left(\pi^*|| \pi_{w_0}\right)+\frac{2 \eta c_{\max }^2|\mathcal{S}||\mathcal{A}|}{(1-\gamma)^4} \left(\sum_{t \in \mathcal{N}_0} \left(\sum_{i=1}^m \lambda^i_t\right)^2 +T-\left|\mathcal{N}_0 \right| \right)\\
\nonumber &+\frac{4}{(1-\gamma)^2} \left(\sum_{t \in \mathcal{N}_0} \sum_{i=1}^m \lambda^i  \left\|Q_{\pi_{w_t}}^i-\bar{Q}_t^i\right\|_2+\sum_{i=m+1}^{m+p} \sum_{t \in \mathcal{N}_i}\left\|Q_{\pi_{w_t}^i}^i-\bar{Q}_t^i\right\|_2\right).
\end{align}

We then first verify item 1. If $\mathcal{N}_0=\emptyset$, then $\sum_{i=1}^p\left|\mathcal{N}_i\right|=T$, and \eqref{eq: summation upper bound td} implies that
\begin{align*}
\beta T &\leq \frac{1}{\eta}\mathbb{E}_{s \sim \nu^*} D_{\mathrm{KL}}\left(\pi^*|| \pi_{w_0}\right)+\frac{2 \eta c_{\max }^2|\mathcal{S}||\mathcal{A}|}{(1-\gamma)^4} \left(\sum_{t \in \mathcal{N}_0} \left(\sum_{i=1}^m \lambda^i_t\right)^2 +T-\left|\mathcal{N}_0 \right| \right)\\
\nonumber &+\frac{4}{(1-\gamma)^2} \left(\sum_{t \in \mathcal{N}_0} \sum_{i=1}^m \lambda^i  \left\|Q_{\pi_{w_t}}^i-\bar{Q}_t^i\right\|_2+\sum_{i=m+1}^{m+p} \sum_{t \in \mathcal{N}_i}\left\|Q_{\pi_{w_t}^i}^i-\bar{Q}_t^i\right\|_2\right),
\end{align*}
which contradicts \eqref{eq: condition for convergence td}. Thus, we must have $\mathcal{N}_0 \neq \emptyset$.

We then proceed to verify item 2. If $\sum_{t \in \mathcal{N}_0}\sum_{i=1}^m \lambda_t^i \left(f_i\left(\pi^*\right)-f_i\left(\pi_{w_t}\right) \right)\leq 0$, then (b) in item 2 holds. If $\sum_{t \in \mathcal{N}_0}\sum_{i=1}^m \lambda_t^i \left(f_i\left(\pi^*\right)-f_i\left(\pi_{w_t}\right) \right)+ \beta \sum_{i=1}^p\left|\mathcal{N}_i\right| \geq 0$ and suppose that $\left|\mathcal{N}_0\right|<T / 2$, i.e., $\sum_{i=1}^p\left|\mathcal{N}_i\right| \geq T / 2$., then \eqref{eq: summation upper bound td} implies that
\begin{align*}
\frac{1}{2} \beta T & \leq \beta\sum_{i=1}^p\left|\mathcal{N}_i\right| \\
& \leq \frac{1}{\eta}\mathbb{E}_{s \sim \nu^*} D_{\mathrm{KL}}\left(\pi^*|| \pi_{w_0}\right)+\frac{2 \eta c_{\max }^2|\mathcal{S}||\mathcal{A}|}{(1-\gamma)^4} \left(\sum_{t \in \mathcal{N}_0} \left(\sum_{i=1}^m \lambda^i_t \right)^2 +T-\left|\mathcal{N}_0 \right| \right)\\
&+\frac{4}{(1-\gamma)^2} \left( \sum_{t \in \mathcal{N}_0}  \sum_{i=1}^m \lambda^i \left\|Q_{\pi_{w_t}}^i-\bar{Q}_t^i\right\|_2+\sum_{i=m+1}^{m+p} \sum_{t \in \mathcal{N}_i}\left\|Q_{\pi_{w_t}^i}^i-\bar{Q}_t^i\right\|_2\right),
\end{align*}
which contradicts \eqref{eq: condition for convergence td}. Hence, (a) in item 2 holds.
\end{proof}

\begin{theorem} \label{thm: td-learning}
For a given number of iterations $T$ of CR-MOPO algorithm, with the choices of $\eta = \frac{(1-\gamma)^2}{r_{\max} m B_1} \sqrt{\frac{\mathbb{E}_{s \sim \nu^*} D_{\mathrm{KL}}\left(\pi^*|| \pi_{w_0}\right)}{|\mathcal{S}||\mathcal{A}| T}}$ and $\beta=\frac{4mB_1 \sqrt{|\mathcal{S}||\mathcal{A}|}}{(1-\gamma)^2 \sqrt{T}} (r_{\max}\sqrt{\mathbb{E}_{s \sim \nu^*} D_{\mathrm{KL}}\left(\pi^*|| \pi_{w_0}\right)} + 2)$, with a probability at least $1-\delta$, we have
\begin{align}
    &\min _{\boldsymbol{\lambda}^* \in S_m}\left(\boldsymbol{\lambda}^{* \top} \boldsymbol{F}\left(\pi^*\right)-\boldsymbol{\lambda}^{* \top} \boldsymbol{F}\left(\pi_{w_\text{out}}\right)\right) \leq \frac{4mB_1 \sqrt{|\mathcal{S}||\mathcal{A}|}}{B_2(1-\gamma)^2 \sqrt{T}} (r_{\max}\sqrt{\mathbb{E}_{s \sim \nu^*} D_{\mathrm{KL}}\left(\pi^*|| \pi_{w_0}\right)} + 2), \\
    &f_i\left(\pi_{w_{\text{out}}}\right)-c_i  \leq \frac{4mB_1 \sqrt{|\mathcal{S}||\mathcal{A}|}}{(1-\gamma)^2 \sqrt{T}} \left(r_{\max}\sqrt{\mathbb{E}_{s \sim \nu^*} D_{\mathrm{KL}}\left(\pi^*|| \pi_{w_0}\right)} + 2 + \frac{\sqrt{(1-\gamma)^5}}{2mB_1}\right).
\end{align}

\end{theorem}
\begin{proof}
First we show that the given values for $\eta$ and $\beta$ satisfy Lemma \ref{lemma: condition for convergence td} as follows,
\begin{align}
 & \frac{2}{\eta T}\mathbb{E}_{s \sim \nu^*} D_{\mathrm{KL}}\left(\pi^*|| \pi_{w_0}\right)+ \frac{4\eta c_{\max }^2|\mathcal{S}||\mathcal{A}|}{(1-\gamma)^4 T} \left(\sum_{t \in \mathcal{N}_0} \left(\sum_{i=1}^m \lambda^i_t\right)^2 +T-\left|\mathcal{N}_0 \right| \right)\\
\nonumber &+\frac{8}{(1-\gamma)^2 T} \left( \sum_{t \in \mathcal{N}_0} \sum_{i=1}^m \lambda_t^i \left\| Q_{\pi_{w_t}}^i-\bar{Q}_t^i \right\|_2+\sum_{i=m+1}^{m+p} \sum_{t \in \mathcal{N}_i}\left\|Q_{\pi_{w_t}^i}^i-\bar{Q}_t^i\right\|_2\right)\\
\nonumber \leq & \frac{2}{\eta T}\mathbb{E}_{s \sim \nu^*} D_{\mathrm{KL}}\left(\pi^*|| \pi_{w_0}\right)+\frac{4 \eta c_{\max }^2|\mathcal{S}||\mathcal{A}|}{(1-\gamma)^4} \left(m^2B_1^2+1\right)+\frac{8m B_1}{(1-\gamma)^2 T} \left( \sum_{t \in \mathcal{N}_0} \left\| Q_{\pi_{w_t}}^i-\bar{Q}_t^i\right\|_2\right)\\
\nonumber \leq & \frac{2}{\eta T}\mathbb{E}_{s \sim \nu^*} D_{\mathrm{KL}}\left(\pi^*|| \pi_{w_0}\right)+\frac{4 \eta c_{\max }^2|\mathcal{S}||\mathcal{A}|}{(1-\gamma)^4} \left(m^2B_1^2+1\right)+\frac{8m B_1\sqrt{\left|\mathcal{S}\right| \left| \mathcal{A} \right| }}{(1-\gamma)^2 \sqrt{T}}\\
\nonumber \leq & \frac{2}{\eta T}\mathbb{E}_{s \sim \nu^*} D_{\mathrm{KL}}\left(\pi^*|| \pi_{w_0}\right)+\frac{4 \eta c_{\max }^2|\mathcal{S}||\mathcal{A}|}{(1-\gamma)^4} \left(m^2B_1^2+1\right)+\frac{8m B_1 \sqrt{\left|\mathcal{S}\right| \left| \mathcal{A} \right|}}{(1-\gamma)^2 \sqrt{T}}\\
\nonumber < &  \frac{4r_{\max}m B_1 \sqrt{  |\mathcal{S}||\mathcal{A}| \mathbb{E}_{s \sim \nu^*} D_{\mathrm{KL}}\left(\pi^*|| \pi_{w_0}\right)} }{(1-\gamma)^2  \sqrt{T} }+\frac{8m B_1\sqrt{\left|\mathcal{S}\right| \left| \mathcal{A} \right| }}{(1-\gamma)^2 \sqrt{T}} \\
= & \beta
\end{align}
where the last inequality follows from $\eta = \frac{(1-\gamma)^2}{r_{\max} m B_1} \sqrt{\frac{\mathbb{E}_{s \sim \nu^*} D_{\mathrm{KL}}\left(\pi^*|| \pi_{w_0}\right)}{|\mathcal{S}||\mathcal{A}|  T }  }$.
This verifies that the condition in Lemma \ref{lemma: condition for convergence td} is satisfied.

We now consider the convergence rate of the multi-objective optimization. 
By the property for the min operator, it holds that
\begin{align*}
 \min _{\boldsymbol{\lambda}^* \in S_m}\left(\boldsymbol{\lambda}^{* \top} \boldsymbol{F}\left(\pi_{w_\text{out}}\right)-\boldsymbol{\lambda}^{* \top} \boldsymbol{F}\left(\pi^*\right)\right) 
\leq & \frac{\sum_{i=1}^m \lambda_\text{out}^i f^i\left(\pi_\text{out}\right)}{\sum_{i=1}^m\lambda_\text{out}^i}-\frac{ \sum_{i=1}^m \lambda_\text{out}^i f^i\left(\pi^*\right)}{\sum_{i=1}^m \lambda_\text{out}^i}\\
\leq &  \frac{1}{B_2}\sum_{i=1}^m \lambda_\text{out}^i f^i\left(\pi_\text{out}\right)- \sum_{i=1}^m \lambda_\text{out}^i f^i\left(\pi^*\right)\\
 \leq &\frac{1}{B_2 \left| \mathcal{N}_0\right|} \sum_{t \in \mathcal{N}_0}\sum_{i=1}^m \lambda_t^i \left(f_i\left(\pi^*\right)-f_i\left(\pi_{w_t}\right) \right)
\end{align*}
where the second inequality follows from $\sum_{i=1}^m \lambda_t^i \geq B_2$ for $t=1, \ldots, T$ in Assumption \ref{ass: lambd bound}. If $\sum_{t \in \mathcal{N}_0}\sum_{i=1}^m \lambda_t^i \left(f_i\left(\pi^*\right)-f_i\left(\pi_{w_t}\right) \right) \leq 0$, then we have 
\begin{align*}
   \min _{\boldsymbol{\lambda}^* \in S_m}\left(\boldsymbol{\lambda}^{* \top} \boldsymbol{F}\left(\pi_{w_\text{out}}\right)-\boldsymbol{\lambda}^{* \top} \boldsymbol{F}\left(\pi^*\right)\right)  \leq 0.
\end{align*}
If 
\begin{align*}
   \sum_{t \in \mathcal{N}_0}\sum_{i=1}^m \lambda_t^i \left(f_i\left(\pi^*\right)-f_i\left(\pi_{w_t}\right) \right) \geq 0 
\end{align*}
we have 
$\left| \mathcal{N}_0\right| \geq \frac{T}{2}$, which implies the following convergence rate

\begin{align*}
 & \min _{\boldsymbol{\lambda}^* \in S_m}\left(\boldsymbol{\lambda}^{* \top} \boldsymbol{F}\left(\pi_{w_\text{out}}\right)-\boldsymbol{\lambda}^{* \top} \boldsymbol{F}\left(\pi^*\right)\right) \\
\nonumber \leq &  \frac{ 4r_{\max}mB_1 \sqrt{|\mathcal{S}||\mathcal{A}| \mathbb{E}_{s \sim \nu^*} D_{\mathrm{KL}}\left(\pi^*|| \pi_{w_0}\right)} }{(1-\gamma)^2  \sqrt{T} }+\frac{8m B_1\sqrt{m} \left|\mathcal{S}\right| \left| \mathcal{A} \right| }{(1-\gamma)^3 \sqrt{T}}.
\end{align*}
We then proceed to bound the constraints violation. For any 
$i \in \{m+1,\ldots, m+p\}$, it holds that
\begin{align*}
f_i\left(\pi_{w_{\text {out }}}\right)-c_i & = \frac{1}{\left|\mathcal{N}_0\right|} \sum_{t \in \mathcal{N}_0} f_i\left(\pi_{w_t}\right)-c_i \\
& \leq \frac{1}{\left|\mathcal{N}_0\right|} \sum_{t \in \mathcal{N}_0}\left(\bar{f}_i\left(\theta_t^i\right)-c_i\right)+\frac{1}{\left|\mathcal{N}_0\right|} \sum_{t \in \mathcal{N}_0}\left(f_i\left(\pi_{w_t}\right)-\bar{f}_i\left(\theta_t^i\right)\right) \\
& \leq \beta+ \frac{1}{\left|\mathcal{N}_0\right|} \sum_{t=0}^{T-1}\left(f_i\left(\pi_{w_t}\right)-\bar{f}_i\left(\theta_t^i\right)\right) \\
& \leq \beta+ \frac{2}{T}\sum_{t=0}^{T-1} f_i\left(\pi_{w_t}\right)-\bar{f}_i\left(\theta_t^i\right) \\
& \leq \beta+ \frac{2}{T}\sum_{t=0}^{T-1} \norm{ Q_{\pi_{w_t}}^i-\bar{Q}_t^i}_2.
\end{align*}
Under the condition defined in \eqref{eq: condition for convergence td}, we have $\frac{2}{T}\sum_{t=0}^{T-1} \norm{ Q_{\pi_{w_t}}^i-\bar{Q}_t^i}_2 \leq \frac{2\sqrt{(1-\gamma)\left| \mathcal{S}\right|\left|\mathcal{A} \right|}}{\sqrt{T}}$, which gives us convergence rates given in the statement.
\end{proof}

\subsection{Unbiased Q-Estimation Estimator}\label{sec: Q-estimator}
\begin{lemma}\label{lemma: expected lambda Q}
For a given $t \in [T]$, we have the following
\begin{align*}
\norm{\EE \left[\sum_{i=1}^m \lambda^i \left(Q_{\pi_{w_t}}^i-\bar{Q}_t^i\right)\right]}_2 \leq &  \sqrt{\mathbb{V}(\boldsymbol{\lambda}) \sum_{i=1}^m\mathbb{V}(\bar{Q}_t^i)}.
\end{align*}
\end{lemma}

\begin{proof}
\begin{align*}
\norm{\EE \left[\sum_{i=1}^m \lambda^i \left(Q_{\pi_{w_t}}^i-\bar{Q}_t^i\right)\right]}_2
=& \norm{\EE \left[\sum_{i=1}^m \left(\lambda^i - \EE[\lambda^i]\right) \left(Q_{\pi_{w_t}}^i-\bar{Q}_t^i\right)\right] +\sum_{i=1}^m  \EE[\lambda^i] \EE \left[\left(Q_{\pi_{w_t}}^i-\bar{Q}_t^i\right)\right]}_2\\
=& \norm{\EE \left[\sum_{i=1}^m \left(\lambda^i - \EE[\lambda^i]\right) \left(Q_{\pi_{w_t}}^i-\bar{Q}_t^i\right)\right]}_2\\
 \leq & \EE \left[\norm{\sum_{i=1}^m \left(\lambda^i - \EE[\lambda^i]\right) \left(Q_{\pi_{w_t}}^i-\bar{Q}_t^i\right)}_2\right]\\
    \leq & \EE \left[\norm{\sum_{i=1}^m \left(\lambda^i - \EE[\lambda^i]\right) \left(Q_{\pi_{w_t}}^i-\bar{Q}_t^i\right)}_2\right]\\
    \leq & \EE \left[\sum_{i=1}^m\norm{ \left(\lambda^i - \EE[\lambda^i]\right) \left(Q_{\pi_{w_t}}^i-\bar{Q}_t^i\right)}_2\right]\\
    \leq & \EE \left[\sum_{i=1}^m\norm{ \lambda^i - \EE[\lambda^i] }_2\norm{Q_{\pi_{w_t}}^i-\bar{Q}_t^i}_2\right]\\
    \leq & \sum_{i=1}^m \EE \left[ \frac{\beta_t}{2}\norm{ \lambda^i - \EE[\lambda^i] }_2^2 + \frac{1}{2\beta_t}\norm{Q_{\pi_{w_t}}^i-\bar{Q}_t^i}_2^2\right]\\
    = &   \frac{\beta_t}{2}\mathbb{V}(\boldsymbol{\lambda}) + \frac{1}{2\beta_t} \sum_{i=1}^m \mathbb{V}(\bar{Q}_t^i) \\
    \leq &  \sqrt{\mathbb{V}(\boldsymbol{\lambda}) \sum_{i=1}^m\mathbb{V}(\bar{Q}_t^i)}
\end{align*}

\end{proof}

\begin{lemma} [Lemma 2 of \cite{zhou2022convergence}]\label{lemma: variance for lambda }
Under Assumption \ref{ass: lambd bound}, the variance of $\boldsymbol{\lambda}_t \text { in Algorithm } \ref{alg:MOPO-CR} \text { is bounded } \mathbb{V}\left[\boldsymbol{\lambda}_t\right] \leq m^2 B_1^2\left(1-\alpha_t\right)^2 $.
\end{lemma}

\begin{lemma} \label{lemma: condition for convergence}
If 
\begin{align} \label{eq: condition for convergence}
\beta > & \frac{2}{\eta T}\mathbb{E}_{s \sim \nu^*} D_{\mathrm{KL}}\left(\pi^*|| \pi_{w_0}\right)+ \frac{4\eta c_{\max }^2|\mathcal{S}||\mathcal{A}|}{(1-\gamma)^4 T} \left(\EE\left[\sum_{t \in \mathcal{N}_0} \left(\sum_{i=1}^m \lambda^i_t \right)^2\right] +T-\left|\mathcal{N}_0 \right| \right)\\
\nonumber &+\frac{8}{(1-\gamma)^2 T} \left( \sum_{t \in \mathcal{N}_0} \left\| \EE \left[\sum_{i=1}^m \lambda_t^i \left(Q_{\pi_{w_t}}^i-\bar{Q}_t^i\right)\right]\right\|_2+\sum_{i=m+1}^{m+p} \sum_{t \in \mathcal{N}_i}\left\|\EE\left[Q_{\pi_{w_t}^i}^i-\bar{Q}_t^i\right]\right\|_2\right),
\end{align}
then we have the following holds:
\begin{enumerate}
    \item $\mathcal{N}_0 \neq \emptyset$,
    \item One of the following two statements must hold,
    \begin{enumerate}
        \item $ \left| \mathcal{N}_0 \right| \geq \frac{T}{2}$,
        \item $\EE\left[ \sum_{t \in \mathcal{N}_0}\sum_{i=1}^m \lambda_t^i \left(f_i\left(\pi^*\right)-f_i\left(\pi_{w_t}\right) \right)\right]\leq 0$.
    \end{enumerate}
\end{enumerate}
\end{lemma}

\begin{proof}
When $t\in \mathcal{N}_0$, from Lemma \ref{lemma: optimality gap for approximated NPG} we have
\begin{align}\label{eq: t in N0}
    & \EE\left[\sum_{i=1}^m \lambda^i \left(f_i\left(\pi^*\right)-f_i\left(\pi_{w_t}\right) \right)\right] \\
    \nonumber&\leq \frac{1}{\eta} \mathbb{E}_{s \sim \nu^*}\left(\EE\left[D_{\mathrm{KL}}\left(\pi^*|| \pi_{w_t}\right)-D_{\mathrm{KL}}\left(\pi^*|| \pi_{w_{t+1}}\right)\right]\right)+\frac{2 \eta c_{\max }^2|\mathcal{S}||\mathcal{A}|}{(1-\gamma)^4}\EE\left[(\sum_{i=1}^m \lambda^i)^2\right] \\
    \nonumber&+\frac{3}{(1-\gamma)^2} \left\| \EE \left[\sum_{i=1}^m \lambda^i \left(Q_{\pi_{w_t}}^i-\bar{Q}_t^i\right)\right]\right\|_2.
\end{align}
Similarly, when $t\in \mathcal{N}_i (i\neq 0)$ we have
\begin{align}\label{eq: t in Ni}
    \nonumber \EE\left[\left(f_i\left(\pi_{w_t}\right)-f_i\left(\pi^*\right) \right) \right]
    \nonumber & \leq \frac{1}{\eta}\mathbb{E}_{s \sim \nu^*}\left(\EE\left[D_{\mathrm{KL}}\left(\pi^*|| \pi_{w_t}\right)-D_{\mathrm{KL}}\left(\pi^*|| \pi_{w_{t+1}}\right)\right]\right)+\frac{2 \eta c_{\max }^2|\mathcal{S}||\mathcal{A}|}{(1-\gamma)^4}\\
    &+\frac{3}{(1-\gamma)^2} \left\| \EE \left[ \left(Q_{\pi_{w_t}}^i-\bar{Q}_t^i\right)\right]\right\|_2.
\end{align}
Taking the summation of \eqref{eq: t in N0} and \eqref{eq: t in Ni} from $t=0$ to $T-1$ gives
\begin{align} \label{eq: summation bound 1}
&\EE\left[\sum_{t \in \mathcal{N}_0}\sum_{i=1}^m \lambda_t^i \left(f_i\left(\pi^*\right)-f_i\left(\pi_{w_t}\right) \right)+\sum_{i=1}^p \sum_{t \in \mathcal{N}_i}\left(f_i\left(\pi_{w_t}\right)-f_i\left(\pi^*\right) \right)\right]  \\
\nonumber & \leq \frac{1}{\eta}\mathbb{E}_{s \sim \nu^*} D_{\mathrm{KL}}\left(\pi^*|| \pi_{w_0}\right)+\frac{2 \eta c_{\max }^2|\mathcal{S}||\mathcal{A}|}{(1-\gamma)^4} \left(\EE\left[\sum_{t \in \mathcal{N}_0} \left(\sum_{i=1}^m \lambda^i_t \right)^2 \right] +T-\left|\mathcal{N}_0 \right| \right)\\
\nonumber &+\frac{3}{(1-\gamma)^2} \left( \sum_{t \in \mathcal{N}_0} \left\| \EE \left[\sum_{i=1}^m \lambda^i \left(Q_{\pi_{w_t}}^i-\bar{Q}_t^i\right)\right]\right\|_2+\sum_{i=m+1}^{m+p} \sum_{t \in \mathcal{N}_i}\left\|\EE\left[Q_{\pi_{w_t}^i}^i-\bar{Q}_t^i\right]\right\|_2\right) .
\end{align}
Note that when $t \in \mathcal{N}_i(i \neq 0)$, we have $\bar{f}_i\left(w_t\right)>c_i+\beta$ (line 9 in Algorithm 1), which implies that
\begin{align} \label{eq: summation bound 2}
\nonumber f_i\left(\pi_{w_t}\right)-f_i\left(\pi^*\right) & \geq \bar{f}_i\left(w_t^i\right)-f_i\left(\pi^*\right)-\left|\bar{f}_i\left(w_t^i\right)-f_i\left(\pi_{w_t}\right)\right| \\
\nonumber & \geq c_i+\beta-J_i\left(\pi^*\right)-\left|\bar{f}_i\left(w_t^i\right)-f_i\left(\pi_{w_t}\right)\right| \\
& \geq \beta - \left\|Q_{\pi_w}^i-\bar{Q}_t^i\right\|_2.
\end{align}
Substituting \eqref{eq: summation bound 2} into \eqref{eq: summation bound 1} gives
\begin{align}
\nonumber & \EE\left[\sum_{t \in \mathcal{N}_0}\sum_{i=1}^m \lambda_t^i \left(f_i\left(\pi^*\right)-f_i\left(\pi_{w_t}\right) \right)+ \beta \sum_{i=1}^p\left|\mathcal{N}_i\right| - \sum_{i=m+1}^{m+p} \sum_{t \in \mathcal{N}_i}\left\|Q_{\pi_w}^i-\bar{Q}_t^i\right\|_2\right]  \\
\nonumber \leq &\frac{1}{\eta}\mathbb{E}_{s \sim \nu^*} D_{\mathrm{KL}}\left(\pi^*|| \pi_{w_0}\right)+\frac{2 \eta c_{\max }^2|\mathcal{S}||\mathcal{A}|}{(1-\gamma)^4} \left(\EE\left[\sum_{t \in \mathcal{N}_0} \left(\sum_{i=1}^m \lambda^i_t\right)^2\right] +T-\left|\mathcal{N}_0 \right| \right)\\
\nonumber &+\frac{3}{(1-\gamma)^2} \left(\sum_{t \in \mathcal{N}_0} \left\| \EE \left[\sum_{i=1}^m \lambda^i \left(Q_{\pi_{w_t}}^i-\bar{Q}_t^i\right)\right]\right\|_2+\sum_{i=m+1}^{m+p} \sum_{t \in \mathcal{N}_i}\left\|\EE\left[Q_{\pi_{w_t}^i}^i-\bar{Q}_t^i\right]\right\|_2\right),
\end{align}
which implies 
\begin{align} \label{eq: summation upper bound}
& \EE\left[\sum_{t \in \mathcal{N}_0}\sum_{i=1}^m \lambda_t^i \left(f_i\left(\pi^*\right)-f_i\left(\pi_{w_t}\right) \right)+ \beta \sum_{i=1}^p\left|\mathcal{N}_i\right|\right]  \\
\nonumber \leq &\frac{1}{\eta}\mathbb{E}_{s \sim \nu^*} D_{\mathrm{KL}}\left(\pi^*|| \pi_{w_0}\right)+\frac{2 \eta c_{\max }^2|\mathcal{S}||\mathcal{A}|}{(1-\gamma)^4} \left(\EE\left[\sum_{t \in \mathcal{N}_0} \left(\sum_{i=1}^m \lambda^i_t\right)^2\right] +T-\left|\mathcal{N}_0 \right| \right)\\
\nonumber &+\frac{4}{(1-\gamma)^2} \left(\sum_{t \in \mathcal{N}_0} \left\| \EE \left[\sum_{i=1}^m \lambda^i \left(Q_{\pi_{w_t}}^i-\bar{Q}_t^i\right)\right]\right\|_2+\sum_{i=m+1}^{m+p} \sum_{t \in \mathcal{N}_i}\left\|\EE\left[Q_{\pi_{w_t}^i}^i-\bar{Q}_t^i\right]\right\|_2\right).
\end{align}

We then first verify item 1. If $\mathcal{N}_0=\emptyset$, then $\sum_{i=1}^p\left|\mathcal{N}_i\right|=T$, and \eqref{eq: summation upper bound} implies that
\begin{align*}
\beta T &\leq \frac{1}{\eta}\mathbb{E}_{s \sim \nu^*} D_{\mathrm{KL}}\left(\pi^*|| \pi_{w_0}\right)+\frac{2 \eta c_{\max }^2|\mathcal{S}||\mathcal{A}|}{(1-\gamma)^4} \left(\EE\left[\sum_{t \in \mathcal{N}_0} \left(\sum_{i=1}^m \lambda^i_t \right)^2\right] +T-\left|\mathcal{N}_0 \right| \right)\\
&+\frac{4}{(1-\gamma)^2} \left( \sum_{t \in \mathcal{N}_0} \left\| \EE \left[\sum_{i=1}^m \lambda^i \left(Q_{\pi_{w_t}}^i-\bar{Q}_t^i\right)\right]\right\|_2+\sum_{i=m+1}^{m+p} \sum_{t \in \mathcal{N}_i}\left\|\EE\left[Q_{\pi_{w_t}^i}^i-\bar{Q}_t^i\right]\right\|_2\right),
\end{align*}
which contradicts \eqref{eq: condition for convergence}. Thus, we must have $\mathcal{N}_0 \neq \emptyset$.

We then proceed to verify item 2. If $\EE\left[\sum_{t \in \mathcal{N}_0}\sum_{i=1}^m \lambda_t^i \left(f_i\left(\pi^*\right)-f_i\left(\pi_{w_t}\right) \right)\right]\leq 0$, then (b) in item 2 holds. If $\EE\left[\sum_{t \in \mathcal{N}_0}\sum_{i=1}^m \lambda_t^i \left(f_i\left(\pi^*\right)-f_i\left(\pi_{w_t}\right) \right)+ \beta \sum_{i=1}^p\left|\mathcal{N}_i\right|\right] \geq 0$ and suppose that $\left|\mathcal{N}_0\right|<T / 2$, i.e., $\sum_{i=1}^p\left|\mathcal{N}_i\right| \geq T / 2$., then \eqref{eq: summation upper bound} implies that
\begin{align*}
\frac{1}{2} \beta T & \leq \beta\sum_{i=1}^p\left|\mathcal{N}_i\right| \\
& \leq \frac{1}{\eta}\mathbb{E}_{s \sim \nu^*} D_{\mathrm{KL}}\left(\pi^*|| \pi_{w_0}\right)+\frac{2 \eta c_{\max }^2|\mathcal{S}||\mathcal{A}|}{(1-\gamma)^4} \left(\EE\left[\sum_{t \in \mathcal{N}_0} \left(\sum_{i=1}^m \lambda^i_t \right)^2\right] +T-\left|\mathcal{N}_0 \right| \right)\\
&+\frac{4}{(1-\gamma)^2} \left( \sum_{t \in \mathcal{N}_0} \left\| \EE \left[ \sum_{i=1}^m \lambda^i \left(Q_{\pi_{w_t}}^i-\bar{Q}_t^i\right)\right]\right\|_2+\sum_{i=m+1}^{m+p} \sum_{t \in \mathcal{N}_i}\left\|\EE\left[Q_{\pi_{w_t}^i}^i-\bar{Q}_t^i\right]\right\|_2\right),
\end{align*}
which contradicts \eqref{eq: condition for convergence}. Hence, (a) in item 2 holds.
\end{proof}

\begin{theorem}\label{thm: q-estimator}
For a given number of iterations $T$ of CR-MOPO algorithm, with the choices of $\eta = \frac{(1-\gamma)^2}{r_{\max} m B_1} \sqrt{\frac{\mathbb{E}_{s \sim \nu^*} D_{\mathrm{KL}}\left(\pi^*|| \pi_{w_0}\right)}{|\mathcal{S}||\mathcal{A}| T}}$ and $\beta=\frac{4mB_1 \sqrt{|\mathcal{S}||\mathcal{A}|}}{(1-\gamma)^2 \sqrt{T}} (r_{\max}\sqrt{\mathbb{E}_{s \sim \nu^*} D_{\mathrm{KL}}\left(\pi^*|| \pi_{w_0}\right)} + 1)$, we have
\begin{align}
    \mathbb{E}\left[\min _{\boldsymbol{\lambda}^* \in S_m}\left(\boldsymbol{\lambda}^{* \top} \boldsymbol{F}\left(\pi^*\right)-\boldsymbol{\lambda}^{* \top} \boldsymbol{F}\left(\pi_{w_\text{out}}\right)\right)\right] &\leq \frac{4mB_1 \sqrt{|\mathcal{S}||\mathcal{A}|}}{B_2(1-\gamma)^2 \sqrt{T}} (r_{\max}\sqrt{\mathbb{E}_{s \sim \nu^*} D_{\mathrm{KL}}\left(\pi^*|| \pi_{w_0}\right)} + 1) \\
    \mathbb{E}\left[f_i\left(\pi_{w_{\text{out}}}\right)\right]-c_i & \leq \frac{4mB_1 \sqrt{|\mathcal{S}||\mathcal{A}|}}{(1-\gamma)^2 \sqrt{T}} (r_{\max}\sqrt{\mathbb{E}_{s \sim \nu^*} D_{\mathrm{KL}}\left(\pi^*|| \pi_{w_0}\right)} + 1)
\end{align}
\end{theorem}

\begin{proof}
First we show that the given values for $\eta$ and $\beta$ satisfy Lemma \ref{lemma: condition for convergence} as follows,
\begin{figure*}  
    \centering  
\begin{align} 
\nonumber & \frac{2}{\eta T}\mathbb{E}_{s \sim \nu^*} D_{\mathrm{KL}}\left(\pi^*|| \pi_{w_0}\right)+ \frac{4\eta c_{\max }^2|\mathcal{S}||\mathcal{A}|}{(1-\gamma)^4 T} \left(\EE\left[\sum_{t \in \mathcal{N}_0} \left( \sum_{i=1}^m \lambda^i_t \right)^2 \right] +T-\left|\mathcal{N}_0 \right| \right)\\
\nonumber &+\frac{8}{(1-\gamma)^2 T} \left( \sum_{t \in \mathcal{N}_0} \left\| \EE \left[\sum_{i=1}^m \lambda^i \left(Q_{\pi_{w_t}}^i-\bar{Q}_t^i\right)\right]\right\|_2+\sum_{i=m+1}^{m+p} \sum_{t \in \mathcal{N}_i}\left\|\EE\left[Q_{\pi_{w_t}^i}^i-\bar{Q}_t^i\right]\right\|_2\right)\\
\nonumber \leq & \frac{2}{\eta T}\mathbb{E}_{s \sim \nu^*} D_{\mathrm{KL}}\left(\pi^*|| \pi_{w_0}\right)+\frac{4 \eta c_{\max }^2|\mathcal{S}||\mathcal{A}|}{(1-\gamma)^4} \left(m^2B_1^2+1\right)+\frac{8}{(1-\gamma)^2 T} \left( \sum_{t \in \mathcal{N}_0} \left\| \EE \left[\sum_{i=1}^m \lambda^i \left(Q_{\pi_{w_t}}^i-\bar{Q}_t^i\right)\right]\right\|_2\right)\\
\nonumber \leq & \frac{2}{\eta T}\mathbb{E}_{s \sim \nu^*} D_{\mathrm{KL}}\left(\pi^*|| \pi_{w_0}\right)+\frac{4 \eta c_{\max }^2|\mathcal{S}||\mathcal{A}|}{(1-\gamma)^4} \left(m^2B_1^2+1\right)+\frac{8}{(1-\gamma)^2 T} \left( \sum_{t \in \mathcal{N}_0} \sqrt{\mathbb{V}(\boldsymbol{\lambda}) \sum_{i=1}^m\mathbb{V}(\bar{Q}_t^i)} \right)\\
\nonumber \leq & \frac{2}{\eta T}\mathbb{E}_{s \sim \nu^*} D_{\mathrm{KL}}\left(\pi^*|| \pi_{w_0}\right)+\frac{4 \eta c_{\max }^2|\mathcal{S}||\mathcal{A}|}{(1-\gamma)^4} \left(m^2B_1^2+1\right)+\frac{8mB_1}{(1-\gamma)^2 T} \left( \sum_{t \in \mathcal{N}_0} (1-\alpha_t)\sqrt{ \sum_{i=1}^m\mathbb{V}(\bar{Q}_t^i)} \right)\\
\nonumber \leq & \frac{2}{\eta T}\mathbb{E}_{s \sim \nu^*} D_{\mathrm{KL}}\left(\pi^*|| \pi_{w_0}\right)+\frac{4 \eta c_{\max }^2|\mathcal{S}||\mathcal{A}|}{(1-\gamma)^4} \left(m^2B_1^2+1\right)+\frac{4mB_1\sqrt{m} \left|\mathcal{S}\right| \left| \mathcal{A} \right|}{(1-\gamma)^3 T}  \sum_{t \in \mathcal{N}_0} (1-\alpha_t) \\
\nonumber \leq & \frac{2}{\eta T}\mathbb{E}_{s \sim \nu^*} D_{\mathrm{KL}}\left(\pi^*|| \pi_{w_0}\right)+\frac{4 \eta c_{\max }^2|\mathcal{S}||\mathcal{A}|}{(1-\gamma)^4} \left(m^2B_1^2+1\right)+\frac{4mB_1\sqrt{m} \left|\mathcal{S}\right| \left| \mathcal{A} \right| }{(1-\gamma)^3 T} \sum_{t \in [T]} (1-\alpha_t) \\
\nonumber = & \frac{2}{\eta T}\mathbb{E}_{s \sim \nu^*} D_{\mathrm{KL}}\left(\pi^*|| \pi_{w_0}\right)+\frac{4 \eta c_{\max }^2|\mathcal{S}||\mathcal{A}|}{(1-\gamma)^4} \left(m^2B_1^2+1\right)+\frac{4mB_1\sqrt{m}\left|\mathcal{S}\right| \left| \mathcal{A} \right|  }{(1-\gamma)^3 T} \sum_{t \in [T]} (1-\alpha_t) \\
\nonumber = & \frac{2}{\eta T}\mathbb{E}_{s \sim \nu^*} D_{\mathrm{KL}}\left(\pi^*|| \pi_{w_0}\right)+\frac{4 \eta c_{\max }^2|\mathcal{S}||\mathcal{A}|}{(1-\gamma)^4} \left(m^2B_1^2+1\right)+\frac{4mB_1 \sqrt{\left|\mathcal{S}\right| \left| \mathcal{A} \right|} }{(1-\gamma)^2 T} \sum_{t \in [T]} \frac{1}{\sqrt{t}} \\
\nonumber \leq & \frac{2}{\eta T}\mathbb{E}_{s \sim \nu^*} D_{\mathrm{KL}}\left(\pi^*|| \pi_{w_0}\right)+\frac{4 \eta c_{\max }^2|\mathcal{S}||\mathcal{A}|}{(1-\gamma)^4} \left(m^2B_1^2+1\right)+\frac{4mB_1\sqrt{\left|\mathcal{S}\right| \left| \mathcal{A} \right| }}{(1-\gamma)^2 \sqrt{T}}\\
\nonumber \leq & \frac{2}{\eta T}\mathbb{E}_{s \sim \nu^*} D_{\mathrm{KL}}\left(\pi^*|| \pi_{w_0}\right)+\frac{4 \eta c_{\max }^2|\mathcal{S}||\mathcal{A}|}{(1-\gamma)^4} \left(m^2B_1^2+1\right)+\frac{4mB_1\sqrt{ \left|\mathcal{S}\right| \left| \mathcal{A} \right| }}{(1-\gamma)^2 \sqrt{T}}\\
\nonumber < &  \frac{4r_{\max}mB_1 \sqrt{  |\mathcal{S}||\mathcal{A}| \mathbb{E}_{s \sim \nu^*} D_{\mathrm{KL}}\left(\pi^*|| \pi_{w_0}\right)} }{(1-\gamma)^2  \sqrt{T} }+\frac{4mB_1\sqrt{\left|\mathcal{S}\right| \left| \mathcal{A} \right|} }{(1-\gamma)^2 \sqrt{T}} \\
\nonumber = &  \frac{4mB_1 \sqrt{|\mathcal{S}||\mathcal{A}|}}{(1-\gamma)^2 \sqrt{T}} (r_{\max}\sqrt{\mathbb{E}_{s \sim \nu^*} D_{\mathrm{KL}}\left(\pi^*|| \pi_{w_0}\right)} + 1) \\
= & \beta
\end{align}
\end{figure*}
where the last inequality follows from $\eta = \frac{(1-\gamma)^2}{r_{\max} m B_1} \sqrt{\frac{\mathbb{E}_{s \sim \nu^*} D_{\mathrm{KL}}\left(\pi^*|| \pi_{w_0}\right)}{|\mathcal{S}||\mathcal{A}|  T }  }$.
This verifies that the condition in Lemma \ref{lemma: condition for convergence} is satisfied.

We now consider the convergence rate of the multi-objective optimization. By the property for the min operator, it holds that
\begin{align*}
 \mathbb{E}\left[\min _{\boldsymbol{\lambda}^* \in S_m}\left(\boldsymbol{\lambda}^{* \top} \boldsymbol{F}\left(\pi_{w_\text{out}}\right)-\boldsymbol{\lambda}^{* \top} \boldsymbol{F}\left(\pi^*\right)\right)\right] \leq & \mathbb{E}\left[\frac{\sum_{i=1}^m \lambda_\text{out}^i f^i\left(\pi_\text{out}\right)}{\sum_{i=1}^m\lambda_\text{out}^i}-\frac{ \sum_{i=1}^m \lambda_\text{out}^i f^i\left(\pi^*\right)}{\sum_{i=1}^m \lambda_\text{out}^i}\right]\\
\leq &  \frac{1}{B_2}\mathbb{E}\left[\sum_{i=1}^m \lambda_\text{out}^i f^i\left(\pi_\text{out}\right)- \sum_{i=1}^m \lambda_\text{out}^i f^i\left(\pi^*\right)\right]\\
 \leq &\frac{1}{B_2 \left| \mathcal{N}_0\right|} \EE\left[ \sum_{t \in \mathcal{N}_0}\sum_{i=1}^m \lambda_t^i \left(f_i\left(\pi^*\right)-f_i\left(\pi_{w_t}\right) \right)\right]  
\end{align*}
where the second inequality follows from $\sum_{i=1}^m \lambda_t^i \geq B_2$ for $t=1, \ldots, T$ in Assumption \ref{ass: lambd bound}. If $\EE\left[ \sum_{t \in \mathcal{N}_0}\sum_{i=1}^m \lambda_t^i \left(f_i\left(\pi^*\right)-f_i\left(\pi_{w_t}\right) \right)\right] \leq 0$, then we have $\mathbb{E}\left[\min _{\boldsymbol{\lambda}^* \in S_m}\left(\boldsymbol{\lambda}^{* \top} \boldsymbol{F}\left(\pi_{w_\text{out}}\right)-\boldsymbol{\lambda}^{* \top} \boldsymbol{F}\left(\pi^*\right)\right)\right]  \leq 0$. If $\EE\left[ \sum_{t \in \mathcal{N}_0}\sum_{i=1}^m \lambda_t^i \left(f_i\left(\pi^*\right)-f_i\left(\pi_{w_t}\right) \right)\right] \geq 0$, we have $\left| \mathcal{N}_0\right| \geq \frac{T}{2}$, which implies the following convergence rate
\begin{align*}
 \mathbb{E}\left[\min _{\boldsymbol{\lambda}^* \in S_m}\left(\boldsymbol{\lambda}^{* \top} \boldsymbol{F}\left(\pi_{w_\text{out}}\right)-\boldsymbol{\lambda}^{* \top} \boldsymbol{F}\left(\pi^*\right)\right)\right] \nonumber \leq &  \frac{ 4r_{\max}mB_1 \sqrt{|\mathcal{S}||\mathcal{A}| \mathbb{E}_{s \sim \nu^*} D_{\mathrm{KL}}\left(\pi^*|| \pi_{w_0}\right)} }{B_2(1-\gamma)^2  \sqrt{T} }+\frac{4mB_1\sqrt{ \left|\mathcal{S}\right| \left| \mathcal{A} \right| }}{B_2(1-\gamma)^2 \sqrt{T}}.
\end{align*}
We then proceed to bound the constraints violation. For any 
$i \in \{m+1,\ldots, m+p\}$, it holds that
\begin{align*}
\mathbb{E}\left[f_i\left(\pi_{w_{\text {out }}}\right)\right]-c_i & =\mathbb{E}\left[\frac{1}{\left|\mathcal{N}_0\right|} \sum_{t \in \mathcal{N}_0} f_i\left(\pi_{w_t}\right)\right]-c_i \\
& \leq  \mathbb{E}\left[ \frac{1}{\left|\mathcal{N}_0\right|} \sum_{t \in \mathcal{N}_0}\left(\bar{f}_i\left(\theta_t^i\right)-c_i\right)+\frac{1}{\left|\mathcal{N}_0\right|} \sum_{t \in \mathcal{N}_0}\left(f_i\left(\pi_{w_t}\right)-\bar{f}_i\left(\theta_t^i\right)\right) \right] \\
& \leq \beta+ \EE \left[\frac{1}{\left|\mathcal{N}_0\right|} \sum_{t=0}^{T-1}\left(f_i\left(\pi_{w_t}\right)-\bar{f}_i\left(\theta_t^i\right)\right) \right] \\
& \leq \beta+ \frac{2}{T}\sum_{t=0}^{T-1} \EE \left[ f_i\left(\pi_{w_t}\right)-\bar{f}_i\left(\theta_t^i\right) \right] \\
& \leq \beta+ \frac{2}{T}\sum_{t=0}^{T-1} \norm{\EE \left[ Q_{\pi_{w_t}}^i-\bar{Q}_t^i\right] }_2\\
& = \beta
\end{align*}
This completes the proof.
\end{proof}

\section{Details of Experiments}
\label{append:Details-of-Experiments}
\subsection{Environment Settings}
\label{append:environment-settings}
Based on MuJoCo~\cite{brockman2016openai, todorov2012mujoco}, several Safe Multi-Objective environments are developed to evaluate safe MORL algorithms.
As shown in Figure~\ref{fig:cmorl-envs-pic-mujoco-tasks},  where Safe Multi-Objective HalfCheetah~(Figure~\ref{fig:cmorl-envs-pic-mujoco-tasks} (a)), Safe Multi-Objective Hopper~(Figure~\ref{fig:cmorl-envs-pic-mujoco-tasks} (b)), Safe Multi-Objective Humanoid~(Figure~\ref{fig:cmorl-envs-pic-mujoco-tasks} (c)), Safe Multi-Objective Swimmer~(Figure~\ref{fig:cmorl-envs-pic-mujoco-tasks} (d)), Safe Multi-Objective Walker~(Figure~\ref{fig:cmorl-envs-pic-mujoco-tasks} (e)), Safe Multi-Objective Pusher~(Figure~\ref{fig:cmorl-envs-pic-mujoco-tasks} (f)) are introduced to evaluate the effectiveness of our methods. Moreover, as shown in Figure~\ref{fig:safe-multi-task-different-tasks-more-experiments}, we provide more experiments to evaluate the effectiveness of our method. 

 \begin{figure}[htbp!]
 \centering
 \subcaptionbox{}
 {
  \includegraphics[width=0.405\linewidth]{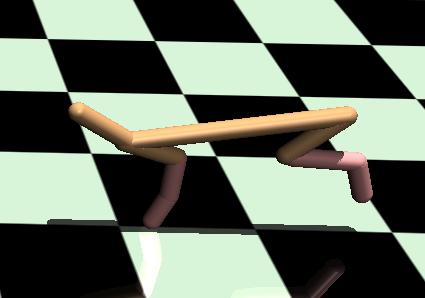}
  }
   \subcaptionbox{}
  {
\includegraphics[width=0.18\linewidth]{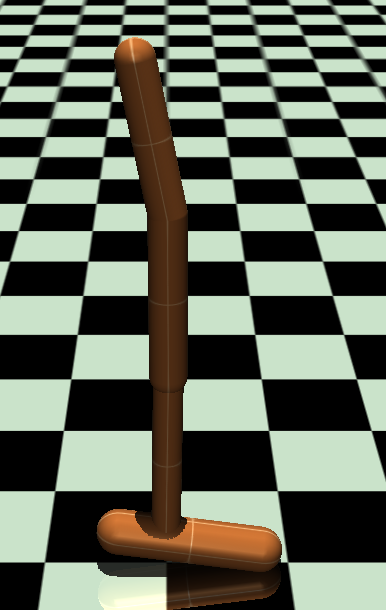}
}
 \subcaptionbox{}
  {
\includegraphics[width=0.2\linewidth]{files/figures/envs_pic/Humanoid-v4.png}
}
 \subcaptionbox{}
  {
\includegraphics[width=0.4\linewidth]{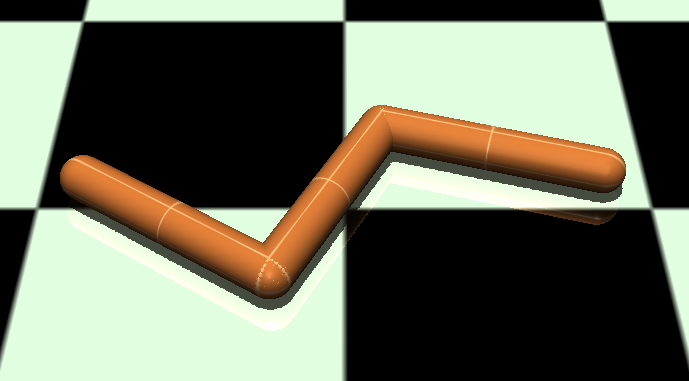}
}
 \subcaptionbox{}
  {
\includegraphics[width=0.2\linewidth]{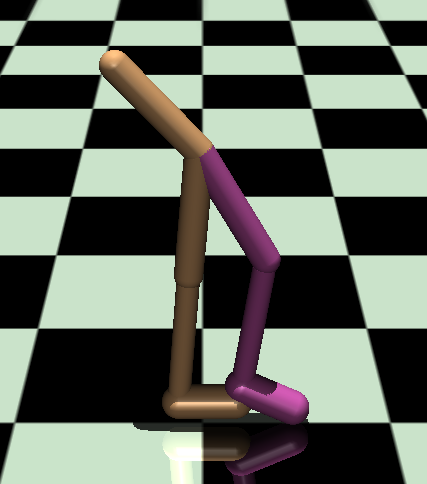}
}
 \subcaptionbox{}
  {
\includegraphics[width=0.25\linewidth]{files/figures/envs_pic/Pusher-v4-03.png}
}
    \vspace{-0pt}
 	\caption{\normalsize Safe Multi-Objective MuJoCo Environments. Specificaly, these environments are Safe Multi-Objective HalfCheetah~(a), Safe Multi-Objective Hopper~(b), Safe Multi-Objective Humanoid~(c), Safe Multi-Objective Swimmer~(d), Safe Multi-Objective Walker~(e) and Safe Multi-Objective Pusher~(f).
 	} 
  \label{fig:cmorl-envs-pic-mujoco-tasks}
 \end{figure}

\begin{figure}[htbp!]
 \centering
 \subcaptionbox{}
 {
  \includegraphics[width=0.21\linewidth]{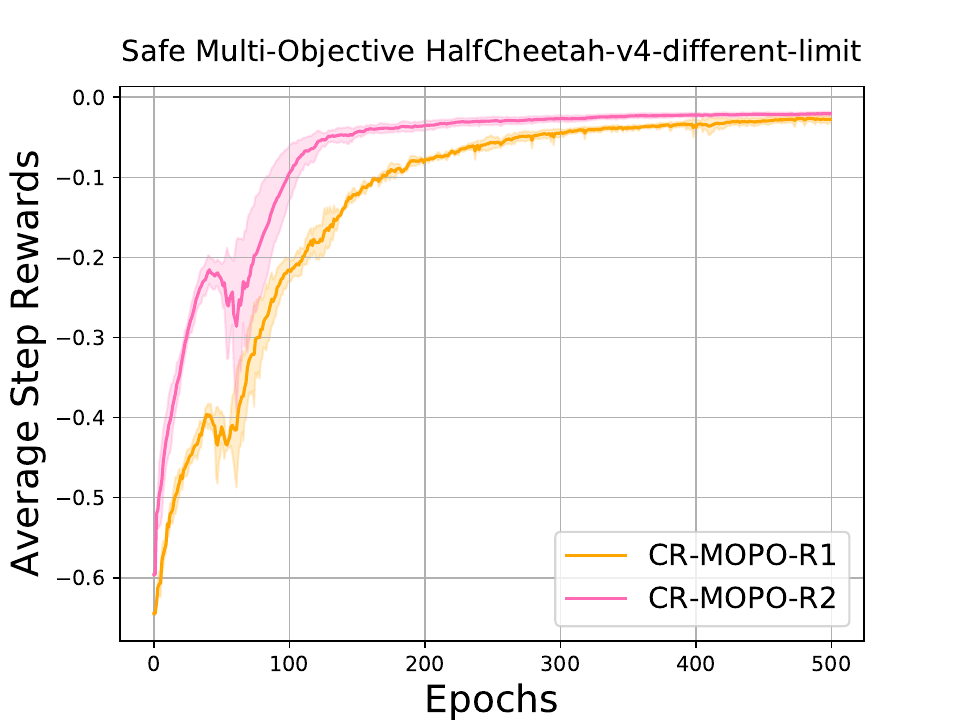}
  }
   \subcaptionbox{}
  {
\includegraphics[width=0.21\linewidth]{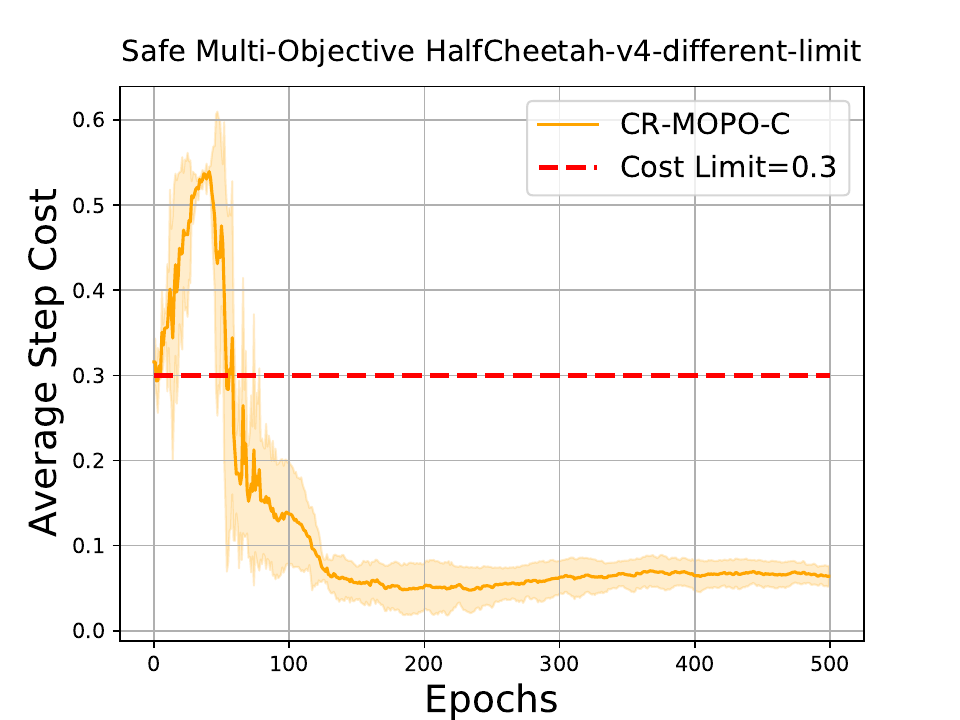}
}
 \subcaptionbox{}
 {
  \includegraphics[width=0.21\linewidth]{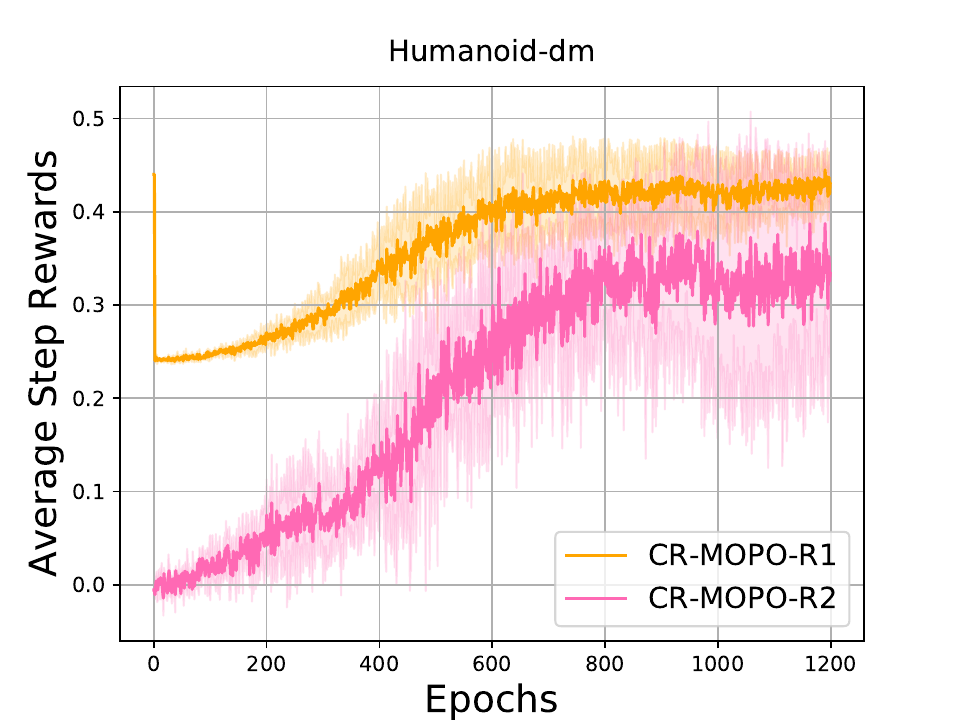}
  }
   \subcaptionbox{}
  {
\includegraphics[width=0.21\linewidth]{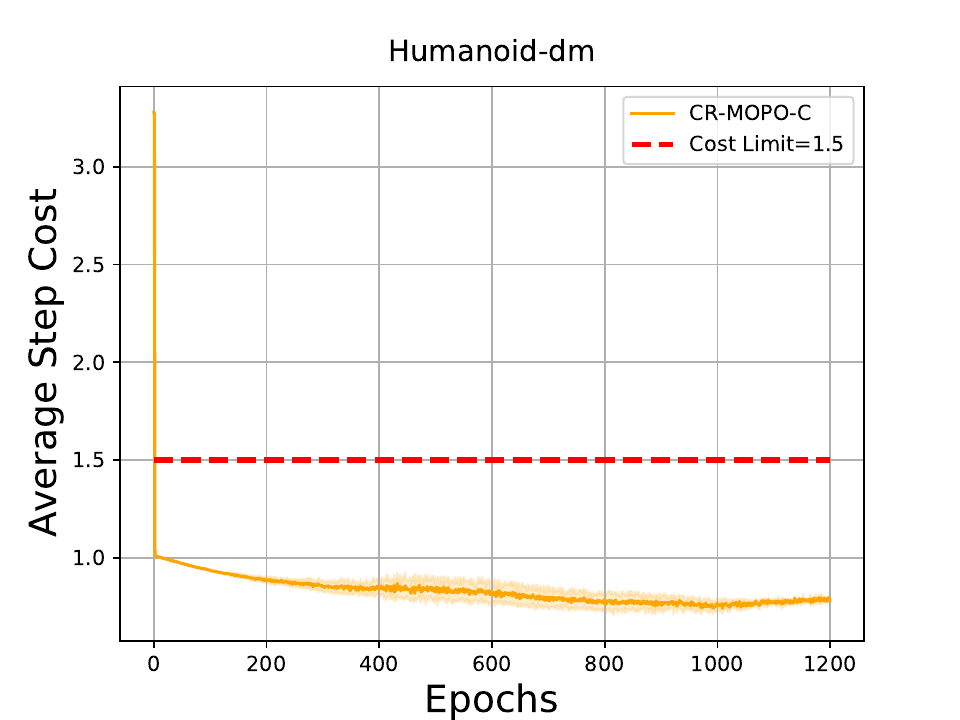}
}
 \subcaptionbox{}
 {
  \includegraphics[width=0.21\linewidth]{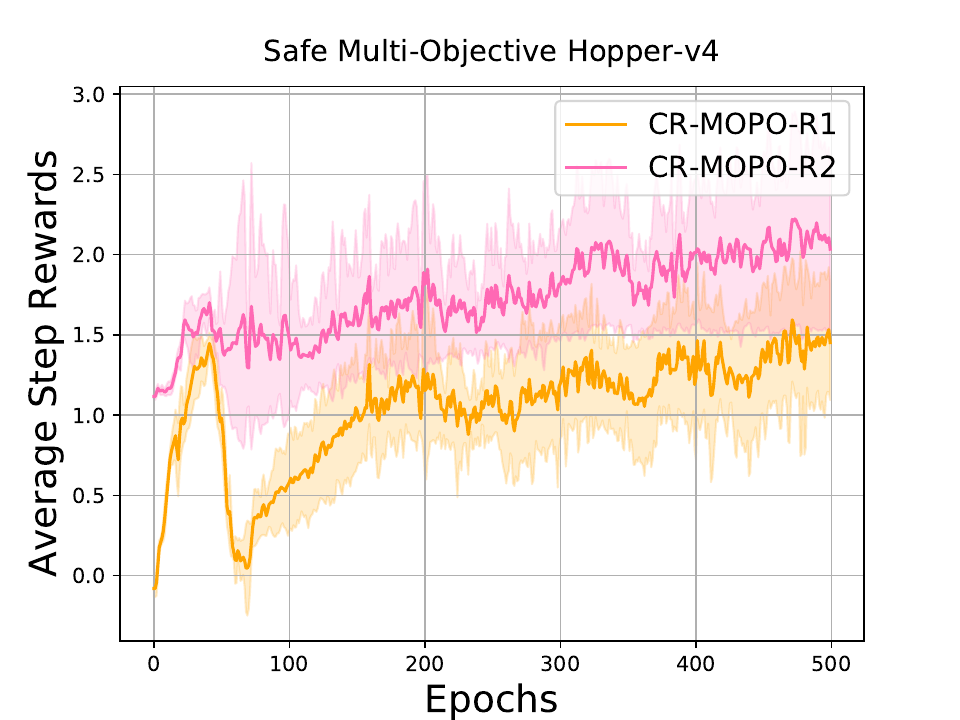}
  }
   \subcaptionbox{}
  {
\includegraphics[width=0.21\linewidth]{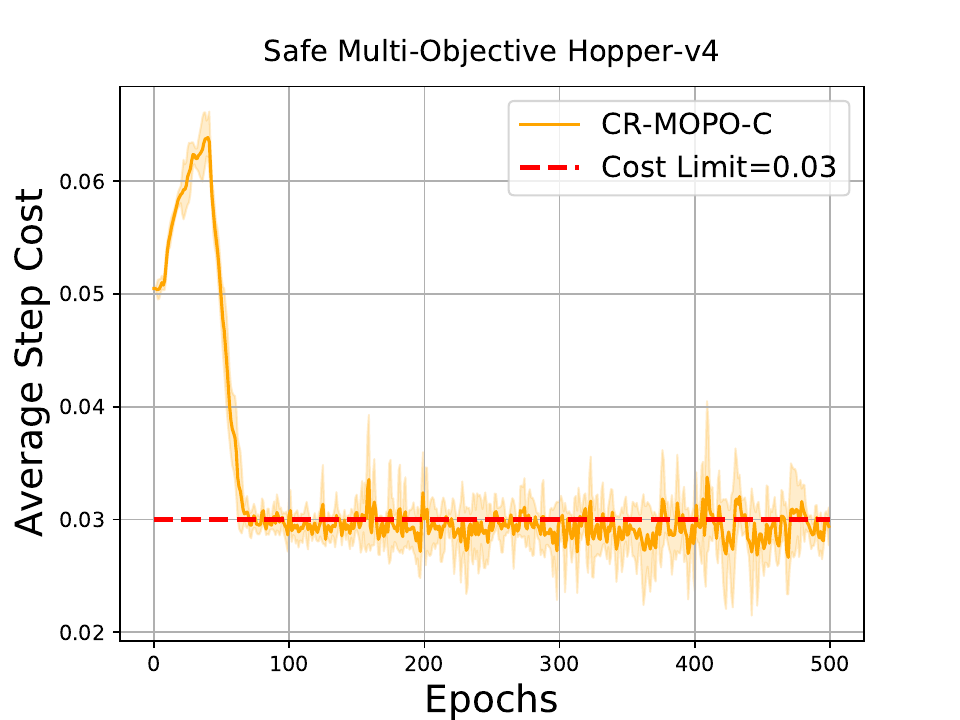}
}
\subcaptionbox{}
 {
  \includegraphics[width=0.21\linewidth]{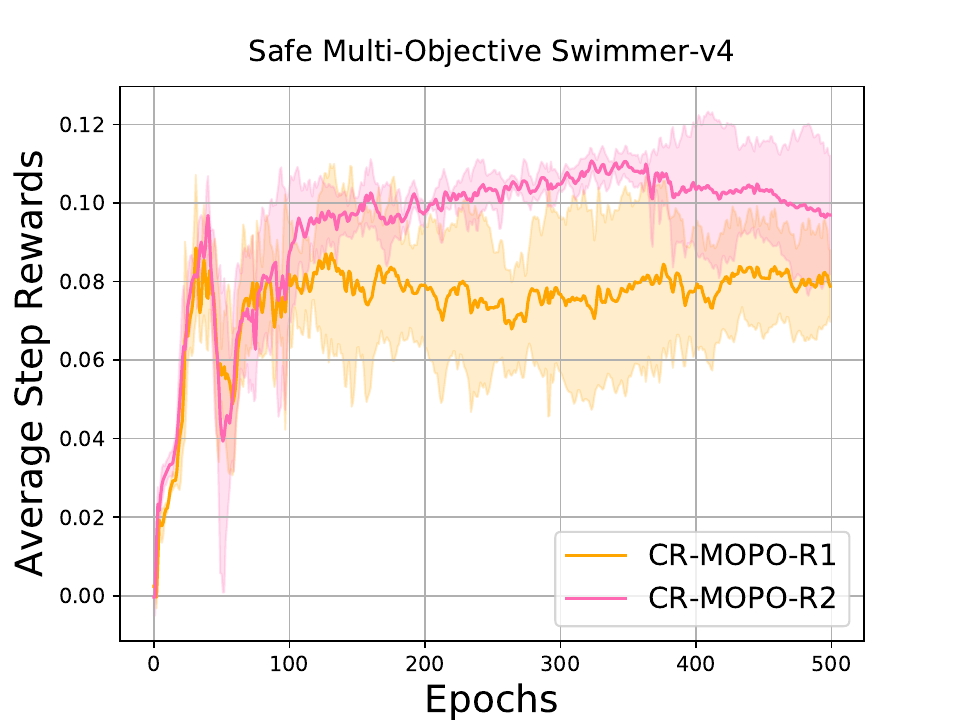}
  }
   \subcaptionbox{}
  {
\includegraphics[width=0.21\linewidth]{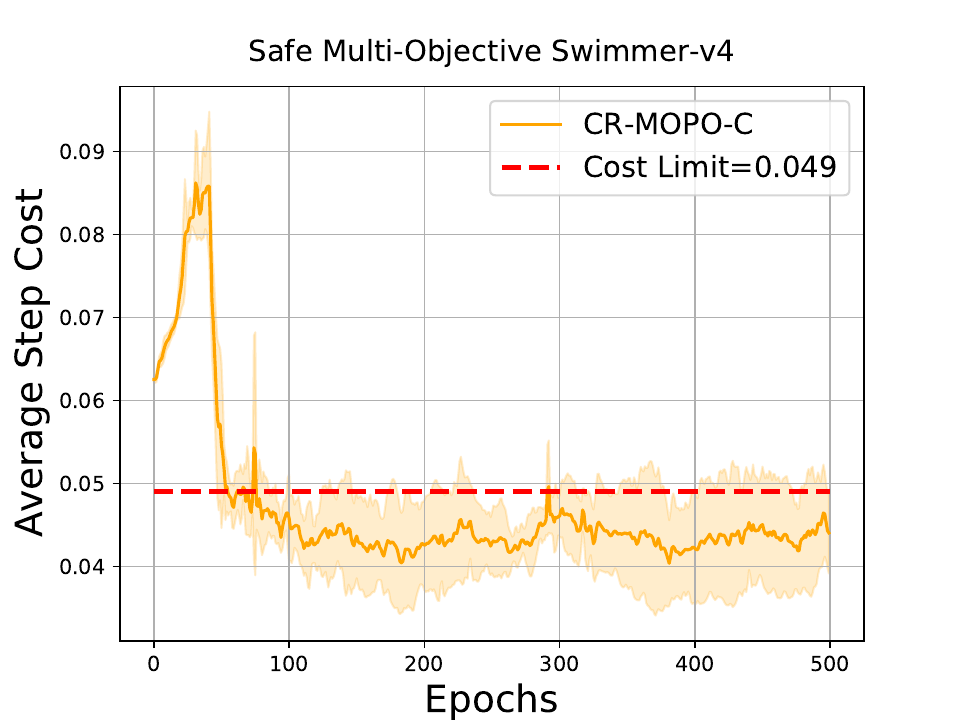}
}
\subcaptionbox{}
 {
  \includegraphics[width=0.21\linewidth]{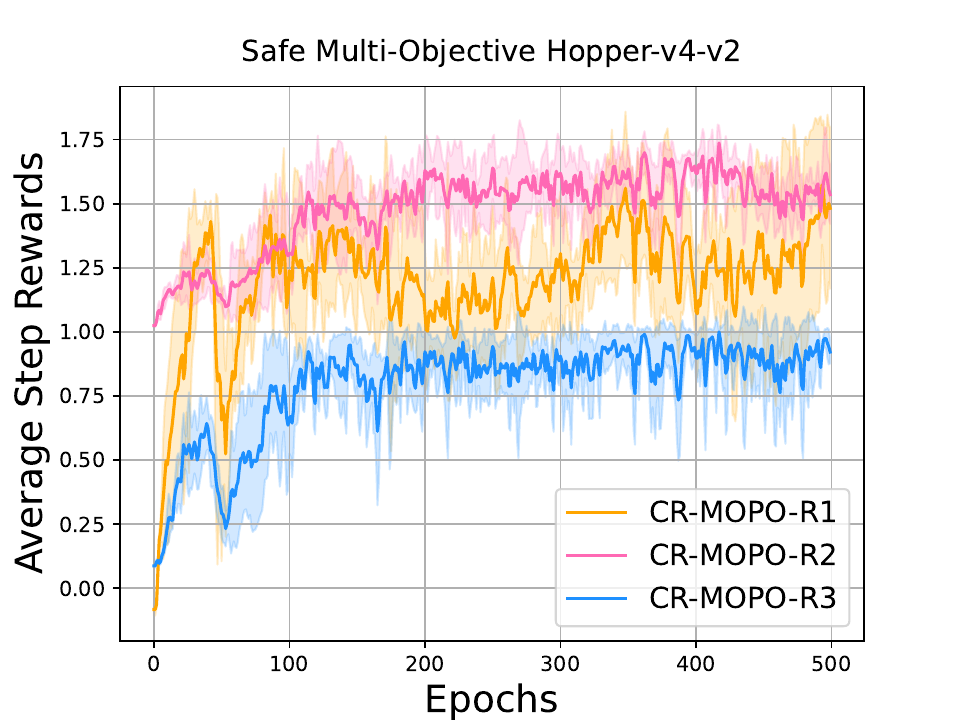}
  }
   \subcaptionbox{}
  {
\includegraphics[width=0.21\linewidth]{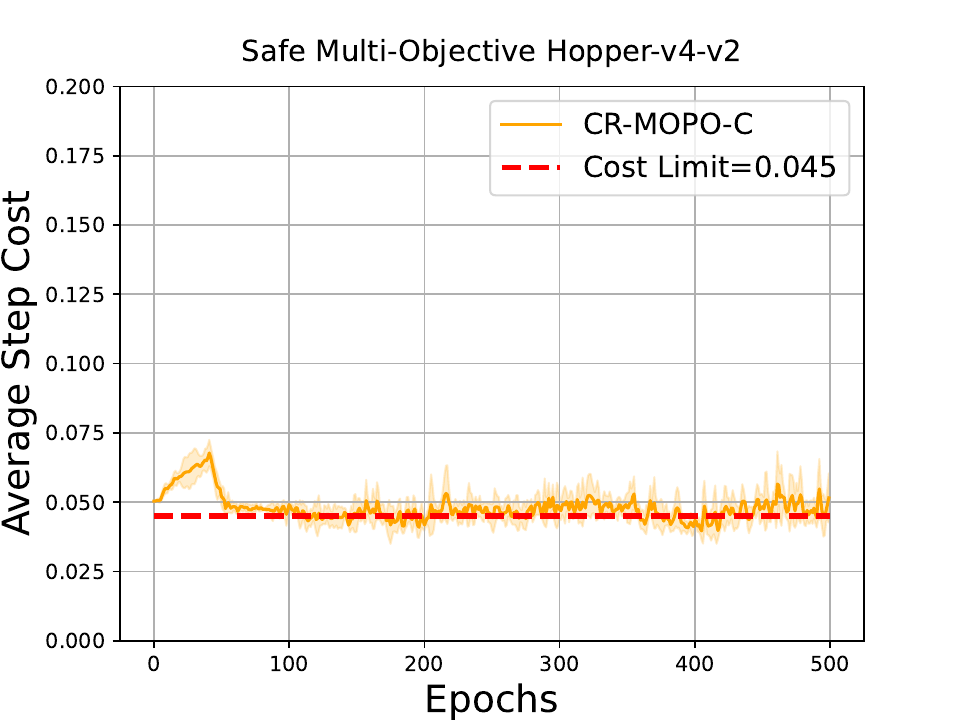}
}
    \vspace{-0pt}
 	\caption{\normalsize More Experiments to evaluate the effectiveness of our method on Safe Multi-Objective MuJoCo environments regarding the reward and safety performance.  
 	} 
  \label{fig:safe-multi-task-different-tasks-more-experiments}
 \end{figure} 

\textbf{Safe Multi-Objective HalfCheetah.} As shown in Figure~\ref{fig:cmorl-envs-pic-mujoco-tasks} (a), the constraint of Safe Multi-Objective halfCheetah is the difference between the robot real-time head height and the robot desired head height, as shown in Equation~(\ref{append-eq:halfcheetah-height}), $C_h^i$ denotes each step's cost value, $H_{cheetah}^i$ denotes the real-time robot head's height, $H_{target}^i$ denotes the target height. We set two reward functions for the robot's two tasks. One reward is Energy saving. The robot needs to save Energy by optimizing its action, which means that to achieve the Goal, the robot needs to use the smallest Energy, $\alpha_a$ denotes the energy weight. As shown in Equation~(\ref{append-eq:halfcheetah-energy}), $R_e^i$ denotes the action $a_{cheetah}^i$  energy at each step $i$. Another reward is velocity; the robot needs to achieve the velocity goal by learning to run as fast as possible. As shown in Equation~(\ref{append-eq:halfcheetah-velocity}), $R_v^i$ denotes each step's velocity reward of the robot, $V_{cheetah}^i$ denotes the real-time each step's velocity, $V_{target}^i$ denotes each step target velocity.

\begin{align}
\label{append-eq:halfcheetah-height}
    C_h^i = |H_{cheetah}^i - H_{target}^i|.
\end{align}

\begin{align}
\label{append-eq:halfcheetah-velocity}
    R_v^i = -|V_{cheetah}^i - V_{target}^i|.
\end{align}

\begin{align}
\label{append-eq:halfcheetah-energy}
    R_e^i = - \alpha_a \Vert \textbf{$a_{cheetah}$}^i \Vert^2.
\end{align}

\textbf{Safe Multi-Objective Hopper.} As shown in Figure~\ref{fig:cmorl-envs-pic-mujoco-tasks} (b), the constraint of Safe Multi-Objective Hopper is Energy that is computed via actions' value, as shown in Equation~(\ref{append-eq:hopper-height}), $C_{e-hopper}^i$ denotes the cost value. There are two objectives in the settings. The first reward function is the forward reward function, as shown in Equation (\ref{append-eq:hopper-velocity}). $R_{f}^i$ denotes the forward reward function, $\alpha_v$ denotes the forward reward weight, $X_v$ denotes the forward velocity. The second reward function is the healthy reward, e.g., the robot angle and state need to satisfy the requirements. As shown in Equation~(\ref{append-eq:hopper-healthy}), $R_{he}^i$ denotes the healthy reward, and $R_{he-state}$ denotes the healthy state. This means the robot state needs to satisfy requirements; $R_z$ is the Z-axis robot healthy; $R_{angle}$ is the robot angle's healthy. That is, the robot angle needs to satisfy the minimum and maximum angle requirements. Moreover, we modify reward settings as three objective tasks, named \textbf{Safe Multi-Objective Hopper-v2}, in which $R_{f}^i, (R_{he-state} + R_z), R_{angle}$ are the three objective rewards respectively.
\begin{align} 
\label{append-eq:hopper-height}
    C_{e-hopper}^i =  -\Vert \textbf{$a_{hopper}$}^i \Vert.
\end{align}

\begin{align}
\label{append-eq:hopper-velocity}
    R_{f}^i = \alpha_v X_V.
\end{align}

\begin{align}
\label{append-eq:hopper-healthy}
    R_{he}^i = R_{he-state} + R_z + R_{angle}.
\end{align}

\begin{equation}
\label{append-eq:hopper-healthy-state}
R_{he-state}=\left\{
\begin{aligned}
1 & , & State_{min} \leq State_{real} \leq State_{max}, \\
0 & , & Others.
\end{aligned}
\right. 
\end{equation}

\begin{equation}
\label{append-eq:hopper-healthy-z-axis}
R_{z}=\left\{
\begin{aligned}
1 & , & Z_{min} \leq Z_{real} \leq Z_{max}, \\
0 & , & Others.
\end{aligned}
\right.
\end{equation}

\begin{equation}
\label{append-eq:hopper-healthy-angle}
R_{angle}=\left\{
\begin{aligned}
1 & , & Angle_{min} \leq Angle_{real} \leq Angle_{max}, \\
0 & , & Others.
\end{aligned}
\right.
\end{equation}

\textbf{Safe Multi-Objective Humanoid.} As shown in Figure~\ref{fig:cmorl-envs-pic-mujoco-tasks} (c), 
the constraint of Safe Multi-Objective Humanoid is the control cost value regarding saving Energy, as shown in Equation~(\ref{append-eq:humanoid-action-cost}), $C_{e-humanoid}^i$ denotes the cost value at each step $i$. The first reward is the forward reward regarding velocity, similar to Equation~(\ref{append-eq:hopper-velocity}). The higher velocity, the better the forward reward. The second reward is the healthy reward, which is about the robot stand. Thus, the higher the axis-z direction height, the better the healthy reward, as shown in Equation~(\ref{append-eq:humanoid-healthy}), $R_{he}^i$ denotes the healthy reward value, $H_z^i$ denotes the X-axis direction height.

\begin{align}
\label{append-eq:humanoid-action-cost}
    C_{e-humanoid}^i =  -\Vert \textbf{$a_{humanoid}$}^i \Vert^2.
\end{align}


\begin{align}
\label{append-eq:humanoid-healthy}
    R_{he}^i = H_z^i.
\end{align}

\textbf{Safe Multi-Objective Swimmer.} As shown in Figure~\ref{fig:cmorl-envs-pic-mujoco-tasks} (d), the constraint is about action Energy, as shown in Equation~\ref{append-eq:swimmer-action-cost}, $C_{e-swimmer}^i$ denotes the cost value at each step, $\alpha_{swimmer}$ denotes the cost weight, \textbf{$a_{swimmer}^i$} denotes the action value. There are two reward functions in the environment. One is the move forward reward, which is about the X-axis direction velocity, as shown in Equation~(\ref{append-eq:swimmer-velocity-X}), $R_{x-v}^i$ denotes the Y-axis direction reward at each step $i$, $\alpha_v$ denotes the weight of X-axis velocity, $X_V$ denotes the X-axis velocity; another is the move left reward, which is about the Y-axis direction velocity, as shown in Equation~(\ref{append-eq:swimmer-velocity-Y}),  $R_{y-v}^i$ denotes the Y-axis direction reward at each step $i$, $Y_V$ denotes the Y-axis velocity. The higher the velocities, the better the reward. 

\begin{align}
\label{append-eq:swimmer-action-cost}
    C_{e-swimmer}^i =  - \alpha_{swimmer} \Vert \textbf{$a_{swimmer}$}^i \Vert.
\end{align}

\begin{align}
\label{append-eq:swimmer-velocity-X}
    R_{x-v}^i = \alpha_v X_V.
\end{align}

\begin{align}
\label{append-eq:swimmer-velocity-Y}
    R_{y-v}^i =  Y_V.
\end{align}

\textbf{Safe Multi-Objective Walker.} As shown in Figure~\ref{fig:cmorl-envs-pic-mujoco-tasks} (e), the constraint of Safe Multi-Objective Walker is about action Energy, as shown in Equation~(\ref{append-eq:walker-action-cost}). The first reward function is about move forward reward, similar to Equation~(\ref{append-eq:swimmer-velocity-X}). The second reward function is about the healthy reward, e.g., healthy angle and healthy Z-axis height. as shown in Equation~(\ref{append-eq:walker-healthy}), and the settings of $R_z$ and $R_{angle}$ are similar to Equation~(\ref{append-eq:hopper-healthy-z-axis}) and Equation~(\ref{append-eq:hopper-healthy-angle}). 

\begin{align}
\label{append-eq:walker-action-cost}
    C_{e-walker}^i =  -\Vert \textbf{$a_{walker}$}^i \Vert.
\end{align}


\begin{align}
\label{append-eq:walker-healthy}
    R_{he}^i = R_z + R_{angle}.
\end{align}



\textbf{Safe Multi-Objective Pusher.} As shown in Figure~\ref{fig:cmorl-envs-pic-mujoco-tasks} (f), the constraint of the Safe Multi-Objective Pusher is about the action Energy, as shown in Equation~(\ref{append-eq:pusher-action-cost}), $C_{e-pusher}^i$ denotes the cost value of each step $i$, \textbf{$a_{pusher}^i$} is the set of actions at each step $i$. The first reward function is about the targeting object to the goal's position, as shown in Equation~(\ref{append-eq:pusher-targeting-object-goal}), $R_{goal-pusher}^i$ denotes the first reward, \textbf{$D_{object-goal}^i$} denotes the distance from targeting object to goal position at each step $i$. The second reward function is about the robot's end effector position to the targeting object's position, as shown in Equation~(\ref{append-eq:pusher-robot-object}), $R_{object-pusher}^i$ denotes the second reward at each step $i$, \textbf{$D_{robot-object}^i$} denotes the distance from the robot to the object at each step $i$. 

\begin{align}
\label{append-eq:pusher-action-cost}
    C_{e-pusher}^i =  \Vert \textbf{$a_{pusher}$}^i \Vert^2.
\end{align}

\begin{align}
\label{append-eq:pusher-targeting-object-goal}
    R_{goal-pusher}^i = -\Vert \textbf{$D_{object-goal}$}^i \Vert.
\end{align}

\begin{align}
\label{append-eq:pusher-robot-object}
     R_{object-pusher}^i = -\Vert \textbf{$D_{robot-object}$}^i \Vert.
\end{align}

\subsection{Implementation Details}
\label{append:implementation-details}

Tables \ref{table:safety-bound} and \ref{table:algorithm-hyparameter-experiments} provide the safety-bound parameters and algorithm parameters used in the study. The server with 40 CPU cores (Intel® Xeon(R) Gold 5218R CPU @ 2.10GHz × 80) and 1 GTX-970 GPU (NVIDIA GeForce GTX 970/PCIe/SSE2)  is used to run the experiments on a Ubuntu 18.04 system.

\begin{table}[!htbp]
 \renewcommand{\arraystretch}{1.2}
  \centering
  \begin{threeparttable}
    \begin{tabular}{cc}
    \toprule
    Environment & value \\
    \midrule
    Safe Multi-Objective HalfCheetah-v4 & 0.1     \\                 Safe Multi-Objective Humanoid-v4 & 0.9    \\
           Safe Multi-Objective Walker-v4 & 0.03            \\ 
           Safe Multi-Objective Humanoid-dm & 1.5             \\ 
           Safe Multi-Objective HalfCheetah-v4-different-limit & 0.3            \\
           Safe Multi-Objective Hopper-v4 & 0.03    \\                 Safe Multi-Objective Swimmer-v4 &  0.049   \\
             Safe Multi-Objective Pusher-v4 & 0.49  \\ 
            Safe Multi-Objective Walker-dm & 1.5   \\ 
            Safe Multi-Objective-v4- soft-v1 & 0.005   \\
    \bottomrule
    \end{tabular}    
    \end{threeparttable}
    \vspace{6pt}
\caption{Safety bound of each step used in the Safe Multi-Objective Environments. }
\label{table:safety-bound}
\end{table}


\begin{table}[!htbp]
 \renewcommand{\arraystretch}{1.2}
  \centering
  \begin{threeparttable}
    \begin{tabular}{cc|cc}
    \toprule
    Parameters & value & Parameters & value \\
    \midrule
    gamma & 0.995             &       tau & 0.97    \\                 l2-reg & 1e-3 & kl &  0.05   \\
           damping & 1e-1          &  batch-size & 16000  \\    
           epoch & 500          &  episode length & 1000  \\     
           grad-c & 0.5          & neural network  & MLP  \\ 
           hidden layer dim & 64          & accept ratio  & 0.1  \\
           energy weight & 1.0          & forward reward weight  & 1.0  \\
    \bottomrule
    \end{tabular}    
    \end{threeparttable}
    \vspace{6pt}
\caption{Key hyparameters used in experiments. }
\label{table:algorithm-hyparameter-experiments}
\end{table}










\end{document}